%% file: iclr2027_conference.tex
\documentclass{article} 
\usepackage{iclr2027_conference,times}

\input{math_commands.tex}

\usepackage{hyperref}
\usepackage{url}

\usepackage[utf8]{inputenc} 
\usepackage[T1]{fontenc}    
\usepackage{hyperref}       
\usepackage{url}            
\usepackage{booktabs}       
\usepackage{amsfonts}       
\usepackage{nicefrac}       
\usepackage{microtype}      
\usepackage{xcolor}         

\usepackage[nolist]{acronym}
\usepackage{mathtools}
\usepackage{amsthm}
\usepackage{algorithm}
\usepackage{algpseudocode}
\usepackage{tikz}
\usepackage{graphicx}
\usepackage{multirow}
\usepackage{placeins}
\usetikzlibrary{positioning, arrows.meta, fit, backgrounds, shapes.geometric, calc}
\input{acronyms}

\newtheorem{proposition}{Proposition}
\newtheorem{lemma}{Lemma}
\newtheorem{remark}{Remark}
\usepackage{enumitem}

\algnewcommand\algorithmicinput{\textbf{Input:}}
\algnewcommand\algorithmicoutput{\textbf{Output:}}
\algnewcommand\Input{\item[\algorithmicinput]}
\algnewcommand\Output{\item[\algorithmicoutput]}

\definecolor{routelilac}{HTML}{c5b0d5} 
\definecolor{routepink}{HTML}{f7b6d2}

\title{Neural Cluster First, Route Second:\\ Capacitated Vehicle Routing via \\Differentiable Optimal Transport}

\author{Samuel J. K. Chin \\
MIT\\
\texttt{jkschin@mit.edu}
\And
Maximilian Schiffer\\
HEC Paris \\
\texttt{schiffer@hec.fr}
}

\iclrfinalcopy 
\begin{document}

\maketitle

\begin{abstract}
The Capacitated Vehicle Routing Problem (CVRP) underpins modern last-mile logistics, where routing decisions recur over the \emph{same fixed service area}, like a city.
In this setting, routing problems share a fixed set of potential customer locations, while active customers and demands vary between instances. 
We study how this spatial support can be exploited through reusable learned representations and design our method around three symmetries of the symmetric Euclidean CVRP: $E(2)$ transformations, vehicle-route permutations, and tour reversal.
We introduce Neural Cluster-First--Route-Second (CFRS), a neural extension of the Fisher--Jaikumar framework that predicts seed-selection scores and customer-to-cluster assignment costs non-autoregressively and respects the three symmetries.
A differentiable entropic optimal transport layer provides capacity-aware supervision and guides discrete capacitated assignment, followed by independent traveling salesman subproblems for route recovery. 
Component ablations show consistent benefits from learned seed selection, while learned assignment costs perform best near the training size and classical FJ costs perform better at larger sizes under exact decoding. 
On the fixed-support distribution with constant capacity, a model trained on $N=100$ achieves a $3.77\%$ routing gap relative to HGS at $N=1000$ without retraining. 
A shallow variant with one attention layer in each transformer achieves a $5.08\%$ gap at this scale, with spatial embeddings consistently improving routing quality over raw coordinates. 
Embedding interpolation further accommodates entirely unseen customer locations without retraining. On standard CVRP benchmarks, a separately trained model achieves a $2.73\%$ routing gap relative to LKH-3 at $N=100$.
\end{abstract}

\section{Introduction}
We study the symmetric Euclidean \ac{CVRP}, in which travel costs are Euclidean distances between locations and are therefore symmetric. Throughout this paper, \ac{CVRP} refers to this setting.
The \ac{CVRP} underpins modern last-mile logistics, from parcel delivery to on-demand grocery and waste collection. 
In many such applications, routing is repeated over the \emph{same fixed service area}: potential customer locations remain fixed while the active customers and their demands change between instances. 
Decades of classical \ac{OR} research have produced highly engineered metaheuristics, notably \ac{LKH-3}~\cite{Helsgaun2017AnProblems} and \ac{HGS}~\cite{Vidal2022HybridNeighborhood,Wouda2024PyVRP:Package}, which deliver near-optimal solutions on individual instances.
However, they typically treat each instance in isolation, without learning from the geography shared across instances.
This leaves an opportunity for learning, where spatial representations of the fixed service area can be learned once and reused across changing customer subsets and demand realizations.

\ac{NCO} offers a way to amortize the cost of solving related instances through a learned model~\cite{Bello2016NeuralLearning,Kool2019AttentionProblems,Vinyals2015PointerNetworks}.
However, the standard \ac{CVRP} benchmarks sample fresh customer and depot coordinates for each instance~\cite{Kool2019AttentionProblems}.
This evaluates generalization to new point sets, without directly testing the reuse of learned representations for recurring locations.
For routing, the dominant paradigm is \ac{AR} construction, in which a model builds a solution one customer at a time~\cite{Kool2019AttentionProblems,Kwon2020Pomo:Learning,Luo2024NeuralGeneralization}.
Although these methods achieve strong solution quality, each decision depends on the preceding partial solution, limiting decoding parallelism as the number of customers grows.
Methods that re-encode the remaining problem at each decoding step incur additional computational cost~\cite{Drakulic2023BQ-NCO:Optimization,Luo2024NeuralGeneralization}.
These costs motivate \ac{NAR} approaches that predict solution structure in parallel.
A common approach is to predict edge heatmaps, from which routes are recovered through subsequent decoding or search~\citep{Joshi2019AnProblem,Fu2021GeneralizeInstances,Qiu2022DIMES:Problems,Min2023UnsupervisedProblem,Ye2024DeepACO:Optimization,Sun2024Difusco:Optimization}.

The classical \ac{CFRS} paradigm suggests an alternative target for \ac{NAR} prediction.
Under the \ac{CFRS} framework, customers are partitioned into capacity-feasible clusters and route construction is reduced to independent \ac{TSP} subproblems that involve each cluster and the depot.
This decomposition has a long history in \ac{OR}~\cite{Bramel1995AProblems,Fisher1981ARouting,Gillett1974AProblem,Renaud1996AnProblem}, with theoretical support from \citet{Bienstock1993ADemands}, who established asymptotic optimality results for \ac{CFRS} heuristics under their studied assumptions.
The importance of geographic grouping and capacity constraints is also reflected in modern classical methods, including cluster-based optimization and capacity-aware search~\citep{Pessoa2021Branch-Cut-and-PriceUncertainty,Silva2025ClusterProblems,Christiaens2020SlackProblems}.

Recent \ac{NCO} methods have also explored this direction, but their partitioning procedures retain sequential or iterative components.
For instance, TAM~\citep{Hou2023GeneralizeReal-time} employs an \ac{AR} model with feasibility masking to construct clusters sequentially.
GLOP~\citep{Ye2024Glop:Real-time} predicts \ac{NAR} heatmaps but uses sequential sampling to construct feasible clusters, while NCC~\citep{Falkner2023NeuralClustering} learns assignment scores to guide an iterative capacitated $k$-means procedure.
Symmetry provides a further principle for choosing the prediction task. 
Sym-NCO \citep{Kim2022Sym-NCO:Optimization} distinguishes problem symmetries, which preserve optimal solution sets, from solution symmetries, which preserve objective values, and exploits both through reinforcement learning and regularization. 
BQ-NCO \citep{Drakulic2023BQ-NCO:Optimization} reduces equivalent decision states, while INViT \citep{Fang2024INViT:Transformer} incorporates invariant geometric views.
Both methods still retain sequential route construction. 
In the CVRP, route permutations and tour reversal change a sequential representation without changing the solution. 
Joint partition prediction therefore removes these arbitrary orderings from the neural prediction task.

\paragraph{Problem setting.}
We consider recurring instances over a fixed, finite \emph{spatial support} $\mathcal{X}\subset\mathbb{R}^{2}$, with depot $x_0\in\mathcal{X}$.
Each instance $(C,\mathbf{d})\sim\mathbb{P}_{\mathrm{dem}}$ consists of an active customer subset $C\subseteq\mathcal{X}\setminus\{x_0\}$ and its positive demand vector $\mathbf{d}$.
Customers must be served from the depot by at most $K$ vehicles of capacity $Q$, minimizing total Euclidean route length.
Only the active customers and their demands vary across instances; the spatial support remains fixed.
Our goal is to learn a policy that reuses knowledge of this geography across realizations from $\mathbb{P}_{\mathrm{dem}}$.
Appendix~\ref{app:problem} gives the full formulation.

\paragraph{Contributions.}
We study how a fixed spatial support can be reused across routing instances on the support and introduce Neural \ac{CFRS}, a non-autoregressive neural extension of the \ac{FJ} framework.
Our contributions are threefold:
\begin{itemize}[leftmargin=0pt, label={}, itemsep=2pt]
    \item \textbf{(1) Reusable spatial representations for recurring routing.}
    We formulate routing over a fixed set of potential customer locations, with active customers and demands varying across instances. Spatial embeddings are learned once and reused, while interpolation extends the representation to previously unseen customer locations without retraining.

    \item \textbf{(2) A non-autoregressive Fisher--Jaikumar framework with joint symmetry abstraction.}
    We learn seed-selection scores and customer-to-cluster assignment costs through non-autoregressive prediction.
    A differentiable \ac{OT} layer provides capacity-aware supervision and guides discrete capacitated assignment, followed by independent \ac{TSP} subproblems for route recovery.
    We establish $E(2)$ spatial invariance, inter-route permutation invariance, and intra-route traversal invariance under tour reversal for Neural \ac{CFRS}, subject to the stated assumptions. These invariances hold by construction, removing the need for data augmentation to learn them.

    \item \textbf{(3) Empirical findings on reusable representations and learned components.}
    Single-layer ablations show that spatial embeddings consistently improve routing quality over raw coordinates across all evaluated problem sizes and both capacity settings.
    Learned seed selection consistently improves routing quality, while learned assignment costs perform best near the training size and classical \ac{FJ} costs yield better results at larger sizes under exact CAP decoding.
    The full-depth model trained on $N=100$ instances achieves a $3.77\%$ routing gap at $N=1000$ under constant capacity and accommodates entirely unseen customer sets with a $4.08\%$ gap without retraining.
    On standard \ac{CVRP} benchmarks, a separately trained model achieves a $2.73\%$ routing gap relative to LKH-3 at $N=100$.
\end{itemize}

\section{Methodology}
The \ac{CVRP} has two symmetry properties that guide our architecture (formal definitions are deferred to Appendix~\ref{app:symmetry}):
\begin{itemize}[leftmargin=*, label={}]
    \item \textbf{Lemma 1 (Isometric Invariance of the Input Space):}
    For a given symmetric Euclidean CVRP instance, applying the same
    $E(2)$ transformation to all customer and depot locations leaves
    the feasible routing solutions and their costs unchanged.
    Consequently, the set of optimal routing solutions is unchanged.
    \item \textbf{Lemma 2 (Edge-Set Formulation of the Solution Space):} Any feasible routing solution is fundamentally characterized as an unordered set of undirected, edge-disjoint cycles.
\end{itemize}

Together, Lemmas~1--2 guide both the input representation and the prediction target.
For the input, we associate each location in the fixed spatial support with a learnable embedding.
An $E(2)$ transformation (translation, rotation, reflection) of the support leaves these identities unchanged.
Pre-training on distances and neighborhood relations yields node embeddings with geometric information about each location and its relationships with surrounding locations.
The same embeddings can then be reused across instances, with current demands supplied separately.
Section~\ref{sec:inputs} describes how these representations are learned.

For the output, Lemma~2 motivates a partitioning formulation that supports joint prediction of clusters, followed by route construction within each group.
To this end, we follow the \ac{FJ} \citep{Fisher1981ARouting} method: a \ac{CAP} determines the clusters, and a \ac{TSP} solver constructs a route through each cluster and the depot.
Neural \ac{CFRS} learns the seed-selection scores and node-to-cluster assignment costs used in this decomposition.
Solving the \ac{CAP} with these learned costs jointly determines all capacity-feasible clusters.
The final solution is then interpreted as an unordered collection of undirected routes.

\paragraph{\acf{CAP}.}
Given $K$ vehicle seeds, let $\Delta_{ij}$ denote the learned cost of assigning customer $i$ to seed $j$.
Let $Y\in\{0,1\}^{N\times K}$ be the assignment matrix, where $Y_{ij}=1$ indicates that customer $i$ is assigned to vehicle $j$, and $Y^*$ is the true assignment for a given instance.
The CAP minimizes the total learned assignment cost:
\begin{equation}
\label{eqn:cap}
\begin{aligned} 
\min_{Y} \quad & \sum_{i=1}^{N} \sum_{j=1}^{K} \Delta_{ij} Y_{ij} \\ 
\end{aligned}
\end{equation}
subject to the capacity constraints $\sum_{i=1}^{N} q_{i} Y_{ij} \le 1.0$ for all $j \in \{1, \dots, K\}$, and the assignment constraints $\sum_{j=1}^{K} Y_{ij} = 1$ for all $i \in \{1, \dots, N\}$. 
Here, $q_i = d_i / Q$ is the normalized fractional capacity required by customer $i$. 
While exact solvers can recover the optimal discrete clusters, integrating them into the forward pass of the neural network is computationally prohibitive and non-differentiable, severing the gradient flow required to learn $\Delta_{ij}$.

We use a differentiable, capacity-aware \ac{OT} surrogate to train $\Delta$, the assignment costs.
At inference, the discrete CAP produces feasible clusters, after which routes are constructed within the clusters.
The final solution quality is measured by the resulting route length.
Figure~\ref{fig:architecture} gives an overview of the method.
We first describe the spatial representations, then the architecture for seed selection and assignment prediction, and finally the discrete decoding procedure.
Propositions~1--3 establish how the complete pipeline preserves the symmetries identified above.

\subsection{The Spatial Support Prior ($\Psi$)}
\label{sec:inputs}

The fixed spatial support allows node embeddings for each location to be learned before training the clustering model.
We pre-train a \ac{SMAE} over $\mathcal{X}$, using spatial block masking to hide regions of the support.
The model learns to reconstruct pairwise Euclidean distances and neighborhood connectivity in the masked regions.
Its attention uses a $k_{\mathrm{attn}}$-NN mask, while the connectivity objective uses a broader $k_{\mathrm{target}}$-NN graph, with $k_{\mathrm{attn}}<k_{\mathrm{target}}$.
These objectives encourage the embeddings to encode geographic information that can be shared across routing instances.
Both the distances and neighborhood relations are preserved by common $E(2)$ transformations of the support.
Full pre-training details are given in Appendix~B.

After pre-training, we freeze the spatial vocabulary
$\Psi\in\mathbb{R}^{|\mathcal{X}|\times d/2}$.
For an active customer $i$, let $\ell_i$ denote its location index in the fixed support.
We combine its spatial embedding with a projection of its current demand:
\begin{equation}
    x_i = \Psi[\ell_i] \oplus \phi_{\mathrm{demand}}(d_i),
\end{equation}
where $\oplus$ denotes concatenation and
$\phi_{\mathrm{demand}}:\mathbb{R}\to\mathbb{R}^{d/2}$ is a linear projection.
The spatial component is shared across instances, while the demand component changes with the current instance.
The final routing sequence is then assembled by prepending the depot embedding $x_0 \in \mathbb{R}^d$ to the active customer set $\{x_1, \dots, x_N\}$.
We also use a $k$-NN attention mask in the downstream transformers to encourage local geographic interactions.
The mask is constructed from Euclidean distances and therefore preserves $E(2)$ invariance as well.

\begin{figure}[t]
    \centering
    \resizebox{\linewidth}{!}{%
        \input{architecture.tex}%
    }
    \caption{Neural CFRS architecture. Spatial pretraining produces reusable embeddings $\Psi$; raw coordinates may be used instead. ST produces node representations and seed scores, followed by capacity-aware seed selection. CT predicts customer-to-seed costs $\Delta$. An OT layer supplies differentiable supervision and guides CAP decoding through node fixing or sparsification. The CAP decoder produces feasible clusters, followed by independent TSP route recovery.}
    \label{fig:architecture}
\end{figure}

\subsection{Phase 1: Seed Selection}
\label{sec:seed_selection}
The primary objective of \acf{ST} is to identify a subset of nodes to serve as initial cluster centers, or ``seeds'', around which the final vehicle routes are constructed. Architecturally, \ac{ST} is a standard Transformer encoder \citep{Vaswani2017AttentionNeed} augmented with a $k$-NN attention mask. Given the invariant base feature representations $X = \{x_0, x_1, \dots, x_N\}$, where $x_0$ represents the depot node, \ac{ST} computes the seed embeddings $H = \{h_0, h_1, \dots, h_N\}$. From these seed embeddings, we project two distinct task heads for seed selection and contrastive representation learning.

\paragraph{Seed Probability and Contrastive Embedding Head.} A two-layer \ac{MLP} $\phi_{\text{seed}}$ estimates the marginal probability $p_i = \sigma(\phi_{\text{seed}}(h_i))$ of each customer $i \in \{1, \dots, N\}$ acting as a seed.
This projection is supervised using the \ac{BCE} loss, $\mathcal{L}_{\text{seed}}$, with 
ground-truth seed indicators derived from expert solutions by selecting the $m$ most distal nodes from the depot within each route. While any node could theoretically serve as a seed, peripheral nodes are ideal candidates since the round-trip cost to the furthest node is a strong first-order approximation of total route distance.
In addition, using $m>1$ provides several positive candidates per route, reflecting the flexibility in choosing a customer to represent each cluster.

To encourage clustering, we project $h_i$ into a latent contrastive space $z_i = \phi_{\text{con}}(h_i)$, where $\phi_{\text{con}}$ is a two-layer \ac{MLP}.
This yields a set of latent representations $Z = \{z_1, \dots, z_N\}$.
To pull the representations of nodes belonging to the same cluster closer together while pushing apart those from different clusters, we apply a supervised contrastive loss\citep{Khosla2020SupervisedLearning}:
\begin{equation}
\mathcal{L}_{\text{con}} = - \sum_{i=1}^N \frac{1}{|\mathcal{P}(i)|} \sum_{p \in \mathcal{P}(i)} \log \frac{\exp(\text{sim}(z_i, z_p) / \tau)}{\sum_{a \neq i} \exp(\text{sim}(z_i, z_a) / \tau)}    
\end{equation}
where $\mathcal{P}(i)$ contains all other nodes in $i$'s optimal cluster, $\text{sim}(\cdot, \cdot)$ is the cosine similarity, and $\tau$ is a temperature hyperparameter.
The loss of \ac{ST} is thus given by $\mathcal{L}_{\text{ST}} = \mathcal{L}_{\text{seed}} + \mathcal{L}_{\text{con}}$.

\paragraph{Seed Selection.} We estimate the required fleet size $K$ using the standard capacity lower bound, $K = \lceil \sum_{i=1}^N d_i / Q \rceil$, noting packing inefficiencies may occasionally underestimate the optimal $K$ (Appendix~\ref{app:k_plus_one}). To instantiate these $K$ seeds with spatial dispersion (Appendix~\ref{app:cagd}), we iteratively select the unassigned node with the highest $p_i$ to initialize a new seed. We then exhaust vehicle capacity $Q$ by greedily packing the most similar unassigned nodes, measured via distance in the contrastive space $Z$. This repeats until all $K$ seeds are selected. 
Extracting their corresponding embeddings from $H$ forms the seed representation set $C = \{c_1, \dots, c_K\}$. Treating each seed as a geographic anchor effectively initializes the $K$ distinct vehicle routes.

\subsection{Phase 2: Global Node-to-Cluster Assignment}
\label{sec:global_assignment}
The \acf{CT} predicts customer-to-seed compatibility scores and assignment costs ($\Delta$) using a Transformer encoder with a $k$-NN attention mask. During training, we use one reference seed from each expert route and associate its assignment column with that route’s customer group. This aligns predicted and reference columns without a separate matching step. At inference, CT uses the seeds selected in Phase 1. The BCE loss directly supervises compatibility scores, while the OT loss supervises capacity-aware assignment probabilities.

\paragraph{Input Representation and Type Embeddings.}The inputs to \ac{CT} consist of the node representation embeddings obtained from \ac{ST}, $H = \{h_0, h_1, \dots, h_N\}$, augmented with the dynamically selected set of $K$ seed representations $C = \{c_1, \dots, c_K\}$. 
To enable the self-attention mechanism to distinguish between the distinct structural roles of these elements, we inject learnable type embeddings $E_{\text{type}} \in \mathbb{R}^{3 \times d_{\text{model}}}$. Corresponding type embeddings are respectively added to the depot node, customer nodes, and candidate seeds:
\begin{equation}
\tilde{h}_0 = h_0 + E_{\text{type}}^{(0)}, \quad \tilde{h}_i = h_i + E_{\text{type}}^{(1)} \ \forall i \in \{1,\dots,N\}, \quad \tilde{c}_j = c_j + E_{\text{type}}^{(2)} \ \forall j \in \{1,\dots,K\}
\end{equation}
The sequence formed by concatenating these discrete representational sets, $[\tilde{h}_0, \tilde{h}_1, \dots, \tilde{h}_N, \tilde{c}_1, \dots, \tilde{c}_K]$, serves as the initial state for the \ac{CT}.
For the $k$-NN attention mask, each seed inherits the exact masking of the customer node from which it was initialized.

\paragraph{Assignment Scoring Head.} 
After processing the initial sequence through the \ac{CT} and applying $L_2$-normalization over each embedding, the final contextualized sequence is partitioned back into its constituent elements: $[\bar{h}_0, \bar{H}, \bar{C}]$. The outputs consist of the updated depot embedding $\bar{h}_0$, the refined customer embeddings $\bar{H} = \{\bar{h}_1, \dots, \bar{h}_N\}$, and the refined seed embeddings $\bar{C} = \{\bar{c}_1, \dots, \bar{c}_K\}$. 
The clustering assignments are then derived by evaluating the affinity between each customer and seed. 
We define $\Delta_{ij} = \| \bar{h}_i - \bar{c}_j \|_2 \in [0, 2]$ as the latent Euclidean distance between the $L_2$-normalized embeddings of customer $i$ and seed $j$.
The continuous logit compatibility matrix $S \in \mathbb{R}^{N \times K}$ is thus computed via learnable parameters $\gamma$ and $\beta$, such that $s_{ij} = -\text{softplus}(\gamma) \Delta_{ij} + \beta$. 
Crucially, mapping these bounded distances into an adaptable logit space allows the gradients to saturate on confident predictions.
Finally, this logit matrix is supervised via the ground-truth node-to-cluster assignments $Y^* \in \{0, 1\}^{N \times K}$ using a \ac{BCE} loss ($\mathcal{L}_{\text{BCE}}$).

\paragraph{Capacity-aware Supervision via Optimal Transport.}
We use an entropic OT surrogate to provide differentiable, capacity-aware supervision for the discrete \ac{CAP}.
We reframe the $K$ vehicles as source nodes, each possessing a normalized capacity of $1.0$, and the $N$ customers as sink nodes demanding a fractional capacity $q_i = d_i / Q$. 
To enforce exact mass conservation since fleet capacity exceeds demand ($K \ge \sum_{i=1}^{N} q_i$), we introduce a dummy "slack" customer (index 0) with demand $q_0 = K - \sum_{i=1}^{N} q_i$ and $\Delta_{0j}=0$.
Instead of optimizing the discrete $Y_{ij}$, we optimize over a continuous transport plan $\Pi \in \mathbb{R}_{+}^{(N+1) \times K}$.
To make this fully differentiable, we introduce an entropic regularization term $\mathcal{H}(\Pi) = \sum_{i,j} \Pi_{ij} \log \Pi_{ij}$, yielding a strictly convex \ac{OT} formulation:
\[
\min_{\Pi} \sum_{i=0}^{N} \sum_{j=1}^{K} \Delta_{ij} \Pi_{ij} + \epsilon \mathcal{H}(\Pi)
\]
subject to the marginal constraints $\sum_{j} \Pi_{ij} = q_i$ and $\sum_{i} \Pi_{ij} = 1.0$. 
This formulation allows us to approximate the regularized transport plan $\Pi^*$ efficiently via the log-domain \ac{SK} algorithm \citep{Cuturi2013SinkhornTransport, Feydy2019InterpolatingDivergences}.
Since $\Delta$ is strictly bounded in $[0, 2]$, we bypass the need for complex annealing schedules and fix $\epsilon$ to $0.01$ during training and $0.001$ during inference.
By discarding the dummy customer row and row-normalizing the fractional mass by the true customer demand $q_i$, we obtain a soft, capacity-compliant categorical probability distribution over the available vehicles: $\hat{Y}_{ij} = \Pi^*_{ij} / q_i$. This differentiable assignment probability acts as a continuous surrogate for $Y^*$ and is supervised via cross-entropy ($\mathcal{L}_{\text{OT}}$).
The loss of \ac{CT} is thus given by $\mathcal{L}_{\text{CT}} = \mathcal{L}_{\text{BCE}} + \mathcal{L}_{\text{OT}}$.
Although the \ac{OT} costs can be scaled to match the linear \ac{CAP} objective, our chosen surrogate yields better routing performance and we discuss a possible explanation in Appendix~\ref{app:cap-aware-surrogate}.

\paragraph{End-to-End Optimization}
The complete architecture is trained end-to-end and gradients propagate seamlessly from the \ac{CT} back to the \ac{ST} via the representations $\tilde{H}$ and $\tilde{C}$. The loss formulation has two distinct components and is given as follows (omitting scalar weights for clarity):
\begin{equation}
    \mathcal{L}_{\text{total}} = \underbrace{\big( \mathcal{L}_{\text{seed}} + \mathcal{L}_{\text{con}} \big)}_{\text{Phase 1: Seed Loss } (\mathcal{L}_{\text{ST}})} + \underbrace{\big(\mathcal{L}_{\text{BCE}} + \mathcal{L}_{\text{OT}}\big)}_{\text{Phase 2: Clustering Loss } (\mathcal{L}_{\text{CT}})}
\end{equation}

\subsection{Phase 3: \acf{OR} Decoding}
\label{sec:or_decoding}

Recall that our primary objective is to solve the discrete \ac{CAP} (Equation~\eqref{eqn:cap}) using the learned cost matrix $\Delta$. 
While passing $\Delta$ directly to an exact solver can potentially yield optimal solutions, $\bar{Y}_{\text{exact}}$, the problem may become intractable as routing instances scale and may require longer run times to obtain the optimal solution.
Instead, $\hat{Y}$ (derived from $\Pi$) provides a highly informative "hint" that can guide and accelerate the solver in two ways. \textbf{(i) Confident Hard-Assignment (Node Fixing):} We establish a confidence threshold $\tau_{\text{high}}$ (e.g., $0.99$). For any customer $i$ where $\max_j \hat{Y}_{ij} > \tau_{\text{high}}$, we trivially fix the decision variable $\bar{Y}_{ij} = 1$, deducting its demand from the vehicle's remaining capacity $Q$. 
The remaining unassigned nodes form a smaller subproblem, yielding the final solution $\bar{Y}_{\text{fixed}}$. 
If hard assignment makes the problem infeasible within the time limit, we randomly prune 10\% of the assigned nodes and re-solve until a solution is found.
This effectively restricts the exact solver to evaluating only the model's low-confidence boundary states. 
\textbf{(ii) Sinkhorn-Guided Sparsification (Edge Selection):} To reduce the number of binary decision variables in the \ac{CAP}, we threshold the transport plan at $\tau_{\text{low}} = 10^{-4}$. Valid edges for the exact solver are restricted to $\mathcal{E}_{\text{valid}} = \{ (i, j) \mid \hat{Y}_{ij} \ge \tau_{\text{low}} \}$. To preserve graph connectivity and promote feasibility, we augment $\mathcal{E}_{\text{valid}}$ with edges connecting each node to its $k$-nearest vehicle seeds (derived from $\Delta$), where $k = \lceil 0.02N \rceil$.
If necessary, $k$ is iteratively incremented until \ac{CAP} feasibility is restored, yielding the final discrete assignment $\bar{Y}_{\text{select}}$ which partitions the nodes into independent \ac{TSP} subproblems.

\paragraph{Joint symmetry abstraction by construction.}
The symmetry abstraction applies directly to the learning task. For the spatial-embedding model, the input features remain unchanged under $E(2)$ transformations of the support. Reordering seeds permutes predicted and reference assignment columns together, leaving $\mathcal{L}_\text{BCE}$ and $\mathcal{L}_\text{OT}$ unchanged. 
Reversing a tour also leaves the cluster unchanged.
These properties hold independently of discrete \ac{CAP} optimality and remove the need for augmentation to learn the specified symmetries. Propositions~1--3 establish the corresponding representation and solution properties, with the decoding assumptions stated in Appendix~\ref{app:symmetry}.
\begin{itemize}[leftmargin=*, label={}]
    \item \textbf{[P1] $E(2)$ Invariance:}
    Translating, rotating, or reflecting the entire spatial support, with location identities preserved, leaves the returned routing solution unchanged.

    \item \textbf{[P2] Inter-Route Permutation Invariance:}
    Reordering the vehicle seeds leaves the final unordered collection of routes unchanged.

    \item \textbf{[P3] Intra-Route Traversal Invariance:}
    Reversing the traversal direction of any route maintains its undirected representation, and hence the routing solution, unchanged.
\end{itemize}

\section{Experiments}
\label{sec:experiments}
We evaluate the practical value of the resulting learning formulation through three questions: which components of the classical decomposition benefit from learning, whether reusable spatial representations remain useful in shallow routing models, and whether those representations extend to unseen customer locations. Component substitutions, matched representation ablations, and interpolation experiments address these questions, respectively. We additionally evaluate size and capacity generalization and a separately trained model on standard CVRP benchmarks.

\paragraph{Dataset and Training} 
We first construct the spatial support dataset by defining a "city" consisting of one depot and $|\mathcal{X}| = 3000$ potential customer locations uniformly sampled over the unit square; all problem instances are subsequently sub-sampled from this fixed spatial support and following prior work, demand $d$ is drawn from a uniform distribution, $d \sim \mathcal{U}(1, 9)$.
For training, we generate 1 million $N=100$ instances. 
For evaluation, we follow \citep{Drakulic2023BQ-NCO:Optimization, Luo2024NeuralGeneralization} and generate 10,000 test instances for $N = 100$, and 128 test instances for each scale $N \in \{200, 500, 1000\}$. 
All instances are solved with \ac{HGS} to near-optimality using the recommended time limit of $T = 240 \times (N/100)$ seconds per instance \citep{Vidal2022HybridNeighborhood}.
Our model is evaluated as a proxy to real-world settings, where there is a constant vehicle capacity of $Q=50$ across all problem sizes.
In addition, we also train and test our model with the data provided by \citep{Luo2024NeuralGeneralization} and show that our method is competitive even without the spatial support.
Our full training setup can be found in Appendix \ref{app:training-setup}.

\paragraph{Baselines} 
For the constant capacity setting, where we evaluate generalization and the spatial embeddings, we compare primarily against \ac{HGS} \citep{Vidal2022HybridNeighborhood} utilizing the PyVRP package \citep{Wouda2024PyVRP:Package}.
For the standard benchmark setting, we compare against \ac{LKH-3} \citep{Helsgaun2017AnProblems}. 
To validate our architectural paradigm, we evaluate against classical \ac{CFRS} methods (Sweep~\citep{Gillett1974AProblem}, Petal~\citep{Foster1976AnProblem}, Fisher \& Jaikumar~\citep{Fisher1981ARouting}) alongside recent neural clustering approaches \citep{Falkner2023NeuralClustering, Hou2023GeneralizeReal-time, Ye2024Glop:Real-time}. 
Finally, we benchmark against a suite of autoregressive \ac{NCO} solvers (POMO, BQ, LEHD) \citep{Kwon2020Pomo:Learning, Drakulic2023BQ-NCO:Optimization, Luo2024NeuralGeneralization} under the greedy setting.
Unless otherwise specified, Neural \ac{CFRS} uses spatial embeddings $\Psi$ and the BCE loss $L_{\mathrm{BCE}}$ and best-in-class comparisons report the best-performing configuration for each setting.

\begin{table}[t]
    \centering
    \caption{Classical CFRS comparisons, component ablations, and
    robustness to unseen customer locations on the fixed-support
    distribution. Routing gaps (\%) are relative to HGS.
    All methods use LKH-3 for route recovery; component variants
    share the downstream CAP decoder.
    FJ denotes Fisher--Jaikumar.
    Bold gaps identify the best results in the comparison block.}
    \label{tab:cfrs-comparison-ablation}
    \label{tab:unseen-locations}
    \small
    \setlength{\tabcolsep}{3.5pt}
    \resizebox{\linewidth}{!}{%
    \begin{tabular}{@{}ll*{4}{rr}@{}}
        \toprule
        \multicolumn{2}{c}{}
        & \multicolumn{2}{c}{$N=100$}
        & \multicolumn{2}{c}{$N=200$}
        & \multicolumn{2}{c}{$N=500$}
        & \multicolumn{2}{c}{$N=1000$} \\
        \cmidrule(lr){3-4}
        \cmidrule(lr){5-6}
        \cmidrule(lr){7-8}
        \cmidrule(lr){9-10}
        Seed & Costs
        & Gap (\%) & Time (s)
        & Gap (\%) & Time (s)
        & Gap (\%) & Time (s)
        & Gap (\%) & Time (s) \\
        \midrule
        \multicolumn{2}{@{}l}{Sweep}
        & 10.87 & 15
        & 14.01 & 3.5
        & 20.58 & 8.5
        & 25.01 & 18 \\
        \multicolumn{2}{@{}l}{Petal}
        & 5.54 & 210
        & 9.67 & 66
        & 14.83 & 2280
        & 17.18 & 1980 \\
        \midrule
        FJ & FJ
        & 6.58 & 197.13
        & 6.28 & 23.26
        & 7.56 & 812.31
        & 7.77 & 832.22 \\
        Grid & Euclidean
        & 12.76 & 185.63
        & 12.26 & 49.64
        & 10.18 & 719.77
        & 7.27 & 830.94 \\
        Learned & Euclidean
        & 9.48 & 201.88
        & 10.18 & 38.67
        & 7.90 & 708.94
        & 5.22 & 842.46 \\
        Grid & Learned
        & 22.43 & 194.38
        & 19.68 & 11.56
        & 13.99 & 182.82
        & 9.48 & 488.70 \\
        FJ & Learned
        & 6.43 & 188.13
        & 6.77 & 7.75
        & 9.20 & 55.13
        & 8.72 & 279.88 \\
        Learned & FJ
        & 4.27 & 185.57
        & 4.42 & 17.18
        & \textbf{3.62} & 692.72
        & \textbf{2.78} & 839.69 \\
        \midrule
        \textbf{Learned} & \textbf{Learned}
        & \textbf{3.03} & 77.60
        & \textbf{3.88} & 5.41
        & 4.54 & 24.83
        & 7.63 & 359.49 \\
        \midrule
        \multicolumn{2}{@{}l}{5\% unseen}
        & 2.70 & 32.92 & 3.05 & 15.69 & 4.11 & 72.00 & 3.86 & 384.78 \\
        \multicolumn{2}{@{}l}{10\% unseen}
        & 2.63 & 34.82 & 3.00 & 18.24 & 4.06 & 75.51 & 3.83 & 499.56 \\
        \multicolumn{2}{@{}l}{15\% unseen}
        & 2.66 & 33.18 & 3.24 & 21.25 & 4.07 & 71.67 & 3.87 & 388.48 \\
        \multicolumn{2}{@{}l}{20\% unseen}
        & 2.63 & 32.73 & 3.28 & 17.04 & 4.16 & 65.56 & 3.86 & 496.45 \\
        \multicolumn{2}{@{}l}{50\% unseen}
        & 2.72 & 36.03 & 3.53 & 16.26 & 4.13 & 85.69 & 4.04 & 466.02 \\
        \multicolumn{2}{@{}l}{100\% unseen}
        & 2.78 & 34.50 & 3.12 & 15.11 & 4.32 & 68.75 & 4.08 & 495.90 \\
        \bottomrule
    \end{tabular}%
    }

    \smallskip
    \begin{minipage}{\linewidth}
        \footnotesize
        Learned/Learned denotes Neural CFRS and provides the
        $0\%$ unseen reference.
        Reported times are total wall-clock times for 10,000
        instances at $N=100$ and 128 at each larger size.
        Each nonzero unseen fraction evaluates the same model
        and decoder without retraining on 1,000 instances at
        $N=100$ and 128 at each larger size.
        Unseen-location instance sets are generated independently
        conditional on the fixed supports and are not paired
        with the reference set.
    \end{minipage}
\end{table}

\paragraph{Classical comparisons and the contribution of learning.}
Table~\ref{tab:cfrs-comparison-ablation} compares Neural \ac{CFRS} with classical \ac{CFRS} methods and ablates its learned components.
Specifically, we replace learned seed selection with grid-based or \ac{FJ} selection, and learned assignment costs with customer--seed Euclidean distances or depot-aware \ac{FJ} insertion costs.
Implementation details are provided in Appendix~\ref{app:ablate_learned_seeds_and_costs}.
With LKH-3 used throughout for route recovery, Neural \ac{CFRS} achieves the lowest routing gap at $N=\{100, 200\}$ among all ablations.
In particular, it improves upon \ac{FJ}, its closest classical counterpart, reducing both the gap and wall clock time across all problem sizes.
More interestingly, the highest performance for $N=\{500, 1000\}$ was actually achieved by learned seeds and \ac{FJ} costs, which suggests that the learned costs do not scale well as the problem size grows.
The ablations show consistent benefits from learned seed selection, while the advantage of learned assignment costs only shows in small problem sizes.
The learned seeds used in these comparisons are obtained from a jointly trained model, whose assignment losses also update the seed representations. 
Consequently, replacing learned costs at inference does not ablate the assignment supervision used during training and is a promising avenue for further research.

\paragraph{Generalization, shallow models, and ablations.} 
We evaluate generalization to larger instances ($N \ge 200$) under constant capacity ($Q=50$) without retraining and demonstrate the benefits of reusable spatial embeddings in shallow models across problem sizes and capacity settings.
As shown in Table~\ref{tab:main_results}, the full Neural \ac{CFRS} model generalizes robustly to larger problems.
Relying purely on raw coordinates without $\mathcal{L}_{\text{BCE}}$ achieves a highly competitive 3.26\% optimality gap on $N=1000$ instances using hard assignment decoding. 
To isolate the representational power of our spatial embeddings, we evaluated the framework with shallow models. 
A single-layer architecture with only 682K parameters (CT and ST) retains an average routing gap of approximately $5\%$ across the evaluated problem sizes under constant vehicle capacity. 
The ablations in Table~\ref{tab:one-layer} of Appendix~\ref{app:one-layer} show that $\Psi$ consistently improves solution quality over raw coordinates, demonstrating the value of the spatial embeddings.
For the full-depth models under constant capacity, using $\Psi$ with $\mathcal{L}_{BCE}$ yields lower routing gaps at $N \in \{100,200\}$, whereas using raw coordinates without $\mathcal{L}_{BCE}$ performs better at $N \in \{500,1000\}$ (Table~\ref{tab:main_results}).
A similar pattern appears in Table~\ref{tab:cfrs-comparison-ablation}: Neural CFRS achieves the lowest routing gaps at $N \in \{100,200\}$, whereas learned seeds combined with FJ assignment costs perform best at $N \in \{500,1000\}$.
We hypothesize that $\mathcal{L}_{\text{BCE}}$ encourages memorization of assignment patterns from the training distribution, improving performance on similar instances but limiting generalization as problem size increases.
Tables~\ref{tab:ablations} and~\ref{tab:one-layer} additionally evaluate the fixed-support models without retraining under the standard capacity schedule $Q\in\{50,80,100,250\}$ for $N\in\{100,200,500,1000\}$, respectively.
Routing gaps increase substantially under this schedule, indicating that generalization is stronger in the constant-capacity setting.
Full ablation results are provided in Appendix~\ref{app:ablations} and ~\ref{app:one-layer}.

\begin{table}[t]
  \caption{\textbf{Main Results.} \textbf{(a)} Performance under our primary constant capacity operational setting, demonstrating robust zero-shot generalization up to $N=1000$. \textbf{(b)} Performance on the standard benchmarks for CVRP100. Generalization scales ($N \ge 200$) for standard benchmarks are deferred to Appendix \ref{app:sota-comparison}.}
  \label{tab:main_results}
  \vspace{2mm}
  \centering
  
  \resizebox{0.85\textwidth}{!}{%
    \begin{tabular}[t]{@{}c@{}}
      \textbf{(a) Constant Capacity Setting ($Q=50$)} \\[1.5mm]
      \begin{tabular}{@{}clrr@{}}
        \toprule
        & \textbf{Model Variant} & \textbf{Gap} & \textbf{Time} \\
        \midrule
        \multirow{4}{*}{\rotatebox[origin=c]{90}{\textbf{$N=100$}}} 
        & Neural CFRS ($\Psi + \mathcal{L}_{BCE}$)$^E$     & \textbf{3.03}\% & 1.3m \\
        & Neural CFRS ($x,y$ w/o $\mathcal{L}_{BCE}$)$^E$  & 3.28\% & 1.7m \\
        & 1-layer ($\Psi + \mathcal{L}_{BCE}$)$^E$         & 4.82\% & 1.6m \\
        & 1-layer ($x,y$ w/o $\mathcal{L}_{BCE}$)$^S$      & 5.62\% & 1.2m \\
        \midrule
        
        \multirow{4}{*}{\rotatebox[origin=c]{90}{\textbf{$N=200$}}}
        & Neural CFRS ($\Psi + \mathcal{L}_{BCE}$)$^S$     & \textbf{3.79}\% & 5.1s \\
        & Neural CFRS ($x,y$ w/o $\mathcal{L}_{BCE}$)$^E$  & 3.86\% & 6.8s \\
        & 1-layer ($\Psi + \mathcal{L}_{BCE}$)$^S$         & 5.40\% & 6.3s \\
        & 1-layer ($x,y$ w/o $\mathcal{L}_{BCE}$)$^E$      & 6.44\% & 5.72s \\
        \midrule
        
        \multirow{4}{*}{\rotatebox[origin=c]{90}{\textbf{$N=500$}}}
        & Neural CFRS ($\Psi + \mathcal{L}_{BCE}$)$^S$     & 4.46\% & 35s \\
        & Neural CFRS ($x,y$ w/o $\mathcal{L}_{BCE}$)$^E$  & \textbf{4.22}\% & 41s \\
        & 1-layer ($\Psi + \mathcal{L}_{BCE}$)$^H$         & 5.95\% & 1.3m \\
        & 1-layer ($x,y$ w/o $\mathcal{L}_{BCE}$)$^H$      & 7.42\% & 1m \\
        \midrule
        
        \multirow{4}{*}{\rotatebox[origin=c]{90}{\textbf{$N=1000$}}}
        & Neural CFRS ($\Psi + \mathcal{L}_{BCE}$)$^H$     & 3.77\% & 10m \\
        & Neural CFRS ($x,y$ w/o $\mathcal{L}_{BCE}$)$^H$  & \textbf{3.26\%} & 21m \\
        & 1-layer ($\Psi + \mathcal{L}_{BCE}$)$^H$         & 5.32\% & 4.4m \\
        & 1-layer ($x,y$ w/o $\mathcal{L}_{BCE}$)$^H$      & 7.31\% & 2.6m \\
        \bottomrule
      \end{tabular}
    \end{tabular}%
    
    \hspace{4mm} 
    
    \begin{tabular}[t]{@{}c@{}}
      \textbf{(b) Standard Benchmark (CVRP100)} \\[1.5mm]
      \begin{tabular}{@{}lrr@{}}
        \toprule
        \textbf{Method} & \textbf{Gap} & \textbf{Time} \\
        \midrule
        \multicolumn{3}{@{}l}{\textit{\textbf{Classical OR \& Heuristics}}} \\
        ~ LKH-3                & 0.00\% & 12h \\
        ~ Sweep-LKH3           & 16.10\% & 2.5m \\
        ~ Petal                & 8.65\% & 31m \\
        ~ Fisher-Jaikumar-LKH3 & 6.75\% & 18m \\
        \addlinespace
        
        \multicolumn{3}{@{}l}{\textit{\textbf{Neural Autoregressive}}} \\
        ~ POMO (greedy)        & 3.16\% & 4.7s \\
        ~ BQ (greedy)          & 2.99\% & 0.7m \\
        ~ LEHD (greedy)        & 3.65\% & 0.5m \\
        \addlinespace
        
        \multicolumn{3}{@{}l}{\textit{\textbf{Neural Clustering / NAR}}} \\
        ~ Neural Cap. Cluster  & 10.09\% & 5.2m \\
        ~ TAM-LKH3             & 3.21\% & 0.86s \\
        ~ GLOP            & 22.70\% & 16m \\
        \addlinespace
        
        \multicolumn{3}{@{}l}{\textit{\textbf{Neural CFRS (Ours)}}} \\
        ~ Exact Decoding       & \textbf{2.73\%} & 1.4m \\
        ~ Sparsified           & 3.75\% & 1.2m \\
        ~ Hard Assignment      & 3.86\% & 1.2m \\
        \bottomrule
      \end{tabular}
    \end{tabular}%
  } 
  
  \vspace{3mm}
  \parbox{\textwidth}{\scriptsize \textit{Note:} Superscripts $E, S,$ and $H$ in Table (a) indicate the best-performing decoding configuration for that row (Exact, Sparsified, and Hard Assignment, respectively). Models in Panel (a) represent the best configurations identified in our ablation study.}
\end{table}


\paragraph{Extension to entirely unseen customer locations.}
We evaluate the full-depth $\Psi$ and $\mathcal{L}_{BCE}$ model on instances containing $5\%$--$100\%$ unseen customers in the same unit-square region, keeping the depot and capacity $Q=50$ fixed. Each unseen location receives an embedding interpolated from its five nearest support locations using normalized inverse-distance weights. Without retraining or fine-tuning, unrestricted CAP decoding yields mean routing gaps of $2.63\%$--$4.32\%$ relative to HGS. When all customers are unseen, the gaps are $2.78\%$, $3.12\%$, $4.32\%$, and $4.08\%$ at $N=\{100,200,500,1000\}$, respectively (Table~\ref{tab:cfrs-comparison-ablation}). The learned spatial embeddings therefore remain useful beyond direct lookup of previously encountered customer identities. Appendix~\ref{app:unseen-robustness} gives the protocol and discusses the possible regularizing effect of interpolation as a hypothesis.

\paragraph{Performance on standard benchmarks.}
A separately trained Neural \ac{CFRS} model using raw coordinates achieves a mean routing gap of $2.73\%$ relative to LKH-3 on CVRP100 with exact \ac{CAP} decoding (Table~\ref{tab:main_results}).
This is lower than the reported gaps of the greedy POMO, BQ, and LEHD baselines.
The model obtains shorter routes than the LKH-3 reference on 174 instances.
Additional results and baseline protocols are provided in Appendix~\ref{app:sota-comparison}.

\section{Conclusion and Limitations}
We introduced Neural CFRS, a non-autoregressive neural extension of Fisher--Jaikumar for recurring CVRP instances over fixed spatial support that jointly respects $E(2)$ transformations, vehicle-route permutations, and tour reversal.
Reusable spatial embeddings and differentiable OT supervision support learned seed selection and capacitated assignment, followed by conventional route recovery.
Our experiments show that learned seeds improve routing quality, spatial embeddings are useful in shallow architectures, and embedding interpolation accommodates unseen locations without retraining.
The benefits of learned assignment costs depend on problem size: classical FJ costs perform better on larger instances under exact CAP decoding.
Generalization is strongest under constant vehicle capacity and deteriorates when instance size and capacity increase together.
The evaluation uses synthetic Euclidean instances; performance on road networks and other geographic distributions remains untested.
Future work will investigate training objectives that target routing cost through the differentiable assignment layer, and interpolated spatial representations under varying depot locations, vehicle capacities, and geographic distributions.
 
\clearpage
\subsection*{AI use statement}
In this work, we used generative AI tools to provide critical ingredients for proving mathematical claims, assist in the writing of proofs, propose or refine hypotheses, design or provide feedback on research methodology or experiments, implement methods.
We have not used generative AI tools to generate synthetic datasets, help develop theoretical models or conceptual frameworks, assist with translation, clean and reformat dataset, support qualitative and thematic data analysis, interpret results are not applicable to this work.
Additionally, we used generative AI tools to discover research topics or identify gaps, brainstorming, edit a research paper to improve readability, identify relevant literature, suggest a structure for a research paper.
We have reviewed all AI-assisted work. 
For example, LLM-generated code was verified with unit tests and both the code and unit tests were checked.
For proof writing, LLMs assisted with the manipulation of equations and the authors ultimately checked the equations. We further tested the proofs with code.
We take responsibility for the final content of this work,
including text, claims or artifacts produced with the aid of generative AI.

\subsection*{Ethics statement}
We do not foresee and ethics issues with our work.

\subsection*{Reproducibility statement}
The problem setting and instance distribution are specified in Section~3 and Appendix~A.
Spatial pretraining and seed selection are described in Appendices~B--C.
Appendix~D provides the assumptions and proofs for the symmetry results, and Appendix~E explains the capacity-aware OT surrogate.
Appendix~F reports the model architecture, training hyperparameters, hardware, and decoding settings.
Appendices~G--M provide the component ablations, evaluation protocols, baseline implementations, and additional analyses.
Code will be released upon publication.



\bibliography{iclr2027_conference, references}
\bibliographystyle{iclr2027_conference}

\appendix
\input{appendices/app_problem_setting}
\input{appendices/app_smae}
\input{appendices/app_cagd}
\input{appendices/app_formal_proofs_of_symmetry_abstraction}
\input{appendices/app_capacity_aware_surrogate}
\input{appendices/app_training_and_testing_setup}
\input{appendices/app_ablate_learned_seeds_and_costs}
\input{appendices/app_ablations}
\input{appendices/app_single_layer_model}
\input{appendices/app_unseen_robustness}
\input{appendices/app_comparison_with_sota}
\input{appendices/app_one_extra_bin}
\input{appendices/app_sparsification_efficiency}

\end{document}

%% file: math_commands.tex
\usepackage{amsmath,amsfonts,bm}

\def\eqref#1{equation~\ref{#1}}

\def\1{\bm{1}}

\DeclareMathAlphabet{\mathsfit}{\encodingdefault}{\sfdefault}{m}{sl}
\SetMathAlphabet{\mathsfit}{bold}{\encodingdefault}{\sfdefault}{bx}{n}



%% file: acronyms.tex
\begin{acronym}
    \acro{NS}{nucleus sampling}
    \acro{GS}{greedy sampling}
    \acro{TSP}{Traveling Salesman Problem}
    \acro{VRP}{Vehicle Routing Problem}
    \acro{CVRP}{Capacitated Vehicle Routing Problem}
    \acro{OR}{Operations Research}
    \acro{BKS}{best-known solution}
    \acro{AM}{Attention Model}
    \acro{CO}{Combinatorial Optimization}
    \acro{MILP}{Mixed Integer Linear Programming}
    \acro{CV}{Computer Vision}
    \acro{NLP}{Natural Language Processing}
    \acro{CNN}{Convolutional Neural Network}
    \acro{ML}{Machine Learning}
    \acro{HGS}{Hybrid Genetic Search}
    \acro{LNS}{Large Neighborhood Search}
    \acro{GA}{Genetic Algorithm}
    \acro{CW}{Clarke-Wright}
    \acro{LK}{Lin-Kernighan}
    \acro{LKH}{Lin-Kernighan-Helsgaun}
    \acro{LKH-3}{Lin-Kernighan-Helsgaun-3}
    \acro{LLaMA-2}{Large Language Model Meta AI 2}
    \acro{DL}{Deep Learning}
    \acro{PN}{Pointer Network}
    \acro{Seq2Seq}{Sequence-to-Sequence}
    \acro{PPO}{Proximal Policy Optimization}
    \acro{LLM}{Large Language Model}
    \acro{BFS}{Breadth-First Search}
    \acro{MCTS}{Monte Carlo Tree Search}
    \acro{AI}{Artificial Intelligence}
    \acro{GPT}{Generative Pre-trained Transformer}
    \acro{ChatGPT}{Chat Generative Pre-trained Transformer}
    \acro{BCP}{Branch-Cut-and-Price}
    \acro{lm-SRC}{Limited Memory Subset Row Cut}
    \acro{SA}{Simulated Annealing}
    \acro{PILS}{Pattern Injection Local Search}
    \acro{KGLS}{Knowledge Guided Local Search}
    \acro{RL}{Reinforcement Learning}
    \acro{POPMUSIC}{Partial Optimization Metaheuristic under Special Intensification Conditions}
    \acro{SISR}{Slack Induction by String Removal}
    \acro{SIFT}{Scale Invariant Feature Transform}
    \acro{HOG}{Histogram of Oriented Gradients}
    \acro{GNN}{Graph Neural Network}
    \acro{NN}{Neural Network}
    \acro{GELU}{Gaussian Error Linear Unit}
    \acro{RELU}{Rectified Linear Unit}
    \acro{T5}{Text-to-Text Transfer Transformer}
    \acro{LM}{Language Model}
    \acro{Prefix LM}{Prefix Language Model}
    \acro{SGD}{Stochastic Gradient Descent}
    \acro{CL}{Curriculum Learning}
    \acro{MLOps}{Machine Learning Operations}
    \acro{RNN}{Recurrent Neural Network}
    \acro{MHA}{Multi-Head Attention}
    \acro{MP}{Mathematical Programming}
    \acro{DRL}{Deep Reinforcement Learning}
    \acro{NCO}{Neural Combinatorial Optimization}
    \acro{SL}{Supervised Learning}
    \acro{HELD}{Heavy Encoder Light Decoder}
    \acro{LEHD}{Light Encoder Heavy Decoder}
    \acro{RCSPP}{resource-constrained shortest path problem}
    \acro{CSPT}{capacitated shortest path tree}
    \acro{FY}{Fenchel-Young}
    \acro{MLP}{Multi-Layer Perceptron}
    \acro{CFRS}{Cluster-First--Route-Second}
    \acro{RFCS}{Route-First--Cluster-Second}
    \acro{CCLP}{Capacitated Concentrator Location Problem}
    \acro{GAP}{Generalized Assignment Problem}
    \acro{LP}{Linear Program}
    \acro{ST}{Seed Transformer}
    \acro{CT}{Clustering Transformer}
    \acro{BCE}{Binary Cross-Entropy}
    \acro{CE}{Cross-Entropy}
    \acro{LR}{learning rate}
    \acro{OT}{Optimal Transport}
    \acro{SK}{Sinkhorn-Knopp}
    \acro{AR}{autoregressive}
    \acro{NAR}{non-autoregressive}
    \acro{CAP}{Capacitated Assignment Problem}
    \acro{NN}{nearest neighbor}
    \acro{SMAE}{Spatial Masked Autoencoder}
    \acro{CAGD}{Capacity Aware Greedy Decoding}
    \acro{FJ}{Fisher-Jaikumar}
\end{acronym}

%% file: architecture.tex
\begin{tikzpicture}[
    >=Latex,
    font=\sffamily\normalsize, 
    module/.style={rectangle, rounded corners=2mm, draw=blue!80!black, fill=blue!5, thick, minimum width=3.2cm, minimum height=1.2cm, align=center},
    vmodule/.style={rectangle, rounded corners=2mm, draw=blue!80!black, fill=blue!5, thick, minimum width=1.4cm, minimum height=3.6cm, align=center},
    htensor/.style={rectangle, rounded corners=0.5mm, draw=black!70, fill=gray!10, thick, minimum width=1.8cm, minimum height=0.45cm, align=center, inner sep=1pt, font=\normalsize},
    faded/.style={rectangle, rounded corners=0.5mm, draw=black!30, fill=gray!5, text=black!40, 
    thick, minimum width=1.8cm, minimum height=0.45cm, align=center, inner sep=1pt, font=\normalsize},
    tensor/.style={rectangle, rounded corners=1mm, draw=black!70, fill=gray!10, thick, minimum width=2.4cm, minimum height=0.8cm, align=center},
    op/.style={rectangle, rounded corners=2mm, draw=green!60!black, fill=green!5, thick, minimum height=0.8cm, align=center},
    filterop/.style={rectangle, rounded corners=2mm, draw=teal!80!black, fill=teal!10, thick, minimum height=0.8cm, align=center},
    solver/.style={rectangle, rounded corners=2mm, draw=orange!80!black, fill=orange!10, thick, minimum width=2.4cm, minimum height=0.8cm, align=center},
    phasebox/.style={draw=black!50, dotted, thick, rounded corners=3mm},
    forward/.style={->, thick, draw=black!80},
    hint/.style={->, thick, dashed, draw=blue!80}
]

\def\yTop{0}
\def\yBot{-4.8} 


\node[htensor] (i0) at (0, 1.3) {DEPOT};
\node[htensor] (i1) at (0, 0.65) {NODE 1};
\node[htensor] (i2) at (0, 0.0) {NODE 2};
\node at (0, -0.55) {$\vdots$};
\node[htensor] (in) at (0, -1.3) {NODE $|\mathcal{X}|$};
\node[fit=(i0)(in), draw=black!50, rounded corners=1mm, inner sep=3pt] (p0_in) {};

\node[vmodule] (smae) at (2.6, 0) {\rotatebox{90}{\begin{tabular}{c}Spatial Masked\\Autoencoder\end{tabular}}};
\node[op, minimum width=2.6cm, minimum height=2.4cm] (psi) at (5.8, 0) {Support Prior\\[1mm]$\Psi \in \mathbb{R}^{(|\mathcal{X}|+1) \times d/2}$};

\draw[forward] (p0_in.east) -- (smae.west);
\draw[forward] (smae.east) -- (psi.west);

\node[htensor] (p1_0) at (10.6, 1.3) {$x_0$};
\node[htensor] (p1_1) at (10.6, 0.65) {$x_1$};
\node[htensor] (p1_2) at (10.6, 0.0) {$x_2$};
\node at (10.6, -0.55) {$\vdots$};
\node[htensor] (p1_n) at (10.6, -1.3) {$x_N$};
\node[fit=(p1_0)(p1_n), draw=black!50, rounded corners=1mm, inner sep=3pt] (p1_in) {};
\node[vmodule] (st) at (13.2, 0) {\rotatebox{90}{Seed Transformer}};

\node[htensor] (h0) at (15.8, 1.3) {$h_0$};
\node[htensor] (h1) at (15.8, 0.65) {$h_1$};
\node[htensor] (h2) at (15.8, 0.0) {$h_2$};
\node at (15.8, -0.55) {$\vdots$};
\node[htensor] (hn) at (15.8, -1.3) {$h_N$};
\node[fit=(h0)(hn), draw=black!50, rounded corners=1mm, inner sep=3pt] (H_box) {};

\node[vmodule] (greedy_op) at (18.4, 0) {\rotatebox{90}{\begin{tabular}{c}Capacity-Aware\\Greedy Decoding\end{tabular}}};
\node[faded] (cg_h0) at (21.0, 1.3) {$h_0$};
\node[htensor, fill=blue!20, draw=blue!80!black] (cg_c1) at (21.0, 0.65) {$c_1$};
\node[faded] (cg_h2) at (21.0, 0.0) {$h_2$};
\node[text=black!40] at (21.0, -0.55) {$\vdots$};
\node[htensor, fill=blue!20, draw=blue!80!black] (cg_ck) at (21.0, -1.3) {$c_K$};
\node[fit=(cg_h0)(cg_ck), draw=black!50, rounded corners=1mm, inner sep=3pt] (cagd_out) {};

\draw[forward] (p1_in.east) -- (st.west);
\draw[forward] (st.east) -- (H_box.west);
\draw[forward] (H_box.east) -- (greedy_op.west);
\draw[forward] (greedy_op.east) -- (cagd_out.west);


\node[htensor] (p2_0) at (0, \yBot+1.5) {$h_0$};
\node[htensor] (p2_1) at (0, \yBot+0.9) {$h_1$};
\node at (0, \yBot+0.4) {$\vdots$};
\node[htensor] (p2_n) at (0, \yBot-0.2) {$h_N$};
\node[htensor, fill=blue!10, draw=blue!80!black] (p2_c1) at (0, \yBot-0.9) {$c_1$};
\node at (0, \yBot-1.4) {$\vdots$};
\node[htensor, fill=blue!10, draw=blue!80!black] (p2_ck) at (0, \yBot-2.0) {$c_K$};
\node[fit=(p2_0)(p2_ck), draw=black!50, rounded corners=1mm, inner sep=3pt] (p2_in) {};

\node[vmodule, minimum height=4.2cm] (ct) at (2.6, \yBot-0.25) {\rotatebox{90}{Clustering Transformer}};
\node[op, inner sep=4pt, minimum width=2.0cm, minimum height=2.4cm] (delta) at (5.2, \yBot-0.25) {Latent Distances\\[1mm]$\Delta \in \mathbb{R}^{N \times K}$};
\node[vmodule, minimum height=4.2cm] (ot_layer) at (7.8, \yBot-0.25) {\rotatebox{90}{\begin{tabular}{c}Optimal Transport\\Layer\end{tabular}}};
\node[op, inner sep=4pt, minimum width=2.8cm, minimum height=1.6cm] (yhat) at (10.4, \yBot-0.25) {Transport\\Probabilities\\[1mm]$\hat{Y}_{ij} = \frac{\Pi^*_{ij}}{q_i}$};
\draw[forward] (p2_in.east |- ct.west) -- node[above=1pt, font=\scriptsize] {$+E_{\text{type}}$} (ct.west);
\draw[forward] (ct.east) -- (delta.west);
\draw[forward] (delta.east) -- (ot_layer.west);
\draw[forward] (ot_layer.east) -- (yhat.west);

\node[op, minimum width=1.8cm] (delta_p3) at (14.2, \yBot-0.25) {$\Delta$};
\node[filterop] (op_hard) at (17.2, \yBot-0.25) {Fix $\hat{Y} > \tau_{\text{high}}$};
\node[filterop] (op_sparse) at (17.2, \yBot-1.45) {Filter $\hat{Y} \ge \tau_{\text{low}}$};
\node[solver] (exact) at (20.0, \yBot+0.95) {Exact MIP};
\node[solver] (reduced) at (20.0, \yBot-0.25) {Reduced MIP};
\node[solver] (sparse) at (20.0, \yBot-1.45) {Sparse MIP};

\node (out_exact) at (22.0, \yBot+0.95) {$\bar{Y}_{\text{exact}}$};
\node (out_reduced) at (22.0, \yBot-0.25) {$\bar{Y}_{\text{fixed}}$};
\node (out_sparse) at (22.0, \yBot-1.45) {$\bar{Y}_{\text{select}}$};

\coordinate (fork) at (15.6, \yBot-0.25);
\draw[thick, draw=black!80] (delta_p3.east) -- (fork);
\draw[forward] (fork) |- (exact.west);
\draw[forward] (fork) -- (op_hard.west);
\draw[forward] (fork) |- (op_sparse.west);
\draw[forward] (exact.east) -- (out_exact.west);
\draw[forward] (op_hard.east) -- (reduced.west);
\draw[forward] (reduced.east) -- (out_reduced.west);
\draw[forward] (op_sparse.east) -- (sparse.west);
\draw[forward] (sparse.east) -- (out_sparse.west);

\begin{scope}[on background layer]
    \draw[phasebox] (-1.3, 2.1) rectangle (22.7, -2.1);
    \draw[draw=black!50, dotted, thick] (8.4, 2.1) -- (8.4, -2.1); 
    
    \node[anchor=south west, font=\bfseries, inner sep=0pt, outer sep=5pt] at (-1.3, 2.1) {Pre-training};
    \node[anchor=south west, font=\bfseries, inner sep=0pt, outer sep=5pt] at (8.6, 2.1) {Phase 1: Seed Generation};

    \draw[phasebox] (-1.3, -2.8) rectangle (22.7, -7.6);
    \draw[draw=black!50, dotted, thick] (12.6, -2.8) -- (12.6, -7.6); 
    
    \node[anchor=south west, font=\bfseries, inner sep=0pt, outer sep=5pt] at (-1.3, -2.8) {Phase 2: Global Assignment};
    \node[anchor=south west, font=\bfseries, inner sep=0pt, outer sep=5pt] at (12.8, -2.8) {Phase 3: OR Decoding};
\end{scope}

\draw[hint] (psi.east) -- node[fill=white, inner sep=2pt, font=\scriptsize] {Sample $\mathbb{P}_{\text{dem}}$} (p1_in.west);

\end{tikzpicture}

%% file: appendices/app_problem_setting.tex
\section{Problem Setting}
\label{app:problem}

In this appendix we state the \ac{CVRP} setting used throughout the paper in
full, complementing the compact formulation given in Section~1.

\paragraph{Spatial support and depot.}
We consider the \ac{CVRP} in a setting that reflects the operational
realities of logistics service providers. Let
$\mathcal{X} \subset \mathbb{R}^2$ denote a fixed, finite set of
potential customer locations, which we refer to as the \emph{spatial
support set}. This set models the structured, static geography of a
service area, e.g., the street network of a city, warehouse districts,
or residential zones. A single depot location $x_0 \in \mathcal{X}$
serves as the start and end point for all routes.

\paragraph{Instances and demand.}
Each problem instance is defined by a set of active customers
$C \subseteq \mathcal{X} \setminus \{x_0\}$, where each customer
$x_i \in C$ has a demand $d_i \in \mathbb{R}_{>0}$. A fleet of $K$
vehicles, each with uniform capacity $Q$, must be routed from the depot
to satisfy all demands while minimizing the total cost, typically
defined as the sum of Euclidean distances traveled.

\paragraph{Stochastic daily demand.}
We model the \emph{stochastic nature of daily demand} via a distribution
$\mathbb{P}_{\text{dem}}$ over finite subsets $C \subseteq \mathcal{X} \setminus \{x_0\}$
and their associated demand vectors $\mathbf{d} = (d_i)_{x_i \in C}$.
That is, each day, a problem instance is drawn as
\[
(C, \mathbf{d}) \sim \mathbb{P}_{\text{dem}}.
\]
Importantly, while the customer subsets and demand realization
$(C, \mathbf{d})$ vary across instances, the spatial support
$\mathcal{X}$ remains fixed. This structure reflects operational regimes
where routing decisions are made repeatedly over the same geographic
area with varying customer demand, e.g., daily parcel delivery in a
city, and is the structural property our learnable spatial-support
embedding (Section~\ref{sec:inputs}) is designed to exploit.

\paragraph{Amortized learning objective.}
In this context, the decision-maker faces a repeated stochastic
optimization problem: on each day $t = 1, 2, \dots$, they must compute
a feasible set of vehicle routes $\pi^{(t)}$ given the realized instance
$(C^{(t)}, \mathbf{d}^{(t)})$, such that all demands are satisfied and
the total cost is minimized. Rather than solving each day's instance
from scratch, our aim is to learn a parametrized one-shot policy
$\pi_\theta$ that amortizes the solution cost over
$\mathbb{P}_{\text{dem}}$ by minimizing the expected routing cost
\[
\theta^\ast \;=\; \arg\min_{\theta}\;
\mathbb{E}_{(C,\mathbf{d}) \sim \mathbb{P}_{\text{dem}}}
\bigl[\,\mathrm{cost}\bigl(\pi_\theta(C,\mathbf{d})\bigr)\,\bigr],
\]
subject to the usual CVRP feasibility constraints: every customer in
$C$ is visited exactly once, each route starts and ends at the depot
$x_0$, and the total demand assigned to any vehicle does not exceed its
capacity $Q$.
Rather than optimizing the expected cost directly, we minimize a surrogate supervised learning objective (cross entropy) derived from a dataset of high-quality expert solutions. 
By training the network to strictly imitate these expert assignments, the model implicitly learns to minimize the underlying Euclidean routing cost.
A learned policy that generalizes across realizations from $\mathbb{P}_{\text{dem}}$ is precisely what our method delivers.

%% file: appendices/app_smae.tex
\section{Spatial Masked Autoencoder Pre-training Details}
\label{app:smae}
Our \ac{SMAE} consists of a standard Transformer encoder \citep{Vaswani2017AttentionNeed} and a masking strategy inspired by BERT \citep{Devlin2019BERT:Understanding}, applied over the fixed spatial support $\mathcal{X}$.

\paragraph{Masking Strategy.} To encourage the model to learn spatial dependencies, we employ a spatial block-masking strategy that obscures between $15\%$ and $30\%$ of the total nodes.
To form these contiguous patches $\mathcal{M} \subset \mathcal{X}$, we iteratively select a random seed node alongside its $5$ to $10$ nearest neighbors, replacing the entire subset with a shared, learnable \texttt{[MASK]} token.

\paragraph{Joint Optimization Objective.}
For the masked tokens, their corresponding output representations $\Psi[i]$ (where $i \in \mathcal{M}$) are projected through two standard two-layer \ac{MLP} heads. 
To prevent trivial identity mappings, we strictly decouple the receptive field from the supervisory signal by forcing the model to reconstruct a broader neighborhood $k_{\text{target}}$ (where $k_{\text{target}} > k_{\text{attn}}$).
\begin{itemize}
\item \textbf{Distance Head:} The first \ac{MLP} predicts a dense continuous distance matrix $\hat{\mathcal{D}} \in \mathbb{R}^{|\mathcal{M}| \times |\mathcal{X}|}$, which is optimized via an exact $L_1$ loss against the ground-truth Euclidean distances.
\item \textbf{Connectivity Head:} Concurrently, the second \ac{MLP} outputs logits to reconstruct the $k_{\text{target}}$-NN connectivity, optimized via a \ac{BCE} loss.
\end{itemize}
By jointly optimizing these heads, the final embeddings explicitly encode both continuous distances and the local topology.

\paragraph{Pre-training Details.} 
The \ac{SMAE} is implemented as a standard Transformer encoder consisting of 6 layers and 4 attention heads, with a hidden dimension of 128 and a feedforward network dimension of 512. 
We establish the local topology by setting the attention neighborhood to $k_{\text{attn}}=20$, while the decoupled reconstruction target is set to $k_{\text{target}}=40$. 
The model is trained for $10,000$ epochs using the Adam optimizer with an initial learning rate of $1\times 10^{-3}$, which is decayed linearly to zero over the total training iterations. 
To ensure robust representation learning, we use a batch size of 8, meaning each forward pass evaluates 8 independently masked realizations of the exact same fixed spatial support. 
Because the architecture operates exclusively on this static global graph, pre-training is highly efficient, converging in less than 30 minutes on a server equipped with NVIDIA L40S GPUs and Intel CPUs. 
Once pre-training is complete, $\Psi$, the frozen base embeddings can be adapted to any downstream routing architecture that operates on this spatial support via a simple linear projection to the required model dimension.

%% file: appendices/app_cagd.tex
\section{Capacity Aware Greedy Decoding}
\label{app:cagd}

As introduced in Phase 1 (Section \ref{sec:seed_selection}), the primary objective of the Capacity Aware Greedy Decoding strategy is to dynamically select a discrete set of $K$ initial cluster centers (seeds) from the candidate customer pool. Rather than selecting seeds purely based on their individual probabilities, this algorithm enforces a spatial dispersion prior while strictly respecting the vehicle capacity $Q$. 

To achieve this, we project the node representations $\mathbf{H}$ through two distinct task heads: the seed probability head $\phi_{\text{seed}}$ yields the marginal probabilities $\mathbf{p}$, and the contrastive head $\phi_{\text{con}}$ yields the latent contrastive representations $\mathbf{Z}$. During each of the $K$ iterations, the algorithm instantiates a new seed by selecting the unassigned node with the highest probability $p_i$. To prevent subsequent seeds from clustering in the same local region, we simulate filling the current seed's vehicle by greedily "packing" the most similar unassigned nodes. This similarity is evaluated using the cosine distance in the contrastive space $\mathbf{Z}$. Nodes are packed until the capacity limit $Q$ is exhausted, at which point they are removed from the candidate pool. 
This ensures that the next selected seed will be geographically distinct. 
In instances where the unassigned candidate pool is exhausted before the sequence is complete—such as when evaluating inflated fleet sizes or when processing in batch mode—the generation loop is terminated early, and the remaining required seed positions are padded to maintain static tensor dimensions for parallel inference.
Algorithm \ref{alg:cagd} details this procedure.

\begin{algorithm}[htbp]
\caption{Capacity Aware Greedy Decoding}
\label{alg:cagd}
\begin{algorithmic}[1]
\Input Node features $\mathbf{H} \in \mathbb{R}^{N \times d}$, demands $\mathbf{d} \in \mathbb{R}^N$, vehicle capacity $Q$.
\Output Set of discrete seed indices $\mathcal{S}$.

\State Compute $K = \lceil \sum_{i=1}^N d_i / Q \rceil$
\State Compute seed probabilities $\mathbf{p}$, where $p_i = \sigma(\phi_{\text{seed}}(\mathbf{h}_i))$
\State Compute contrastive embeddings $\mathbf{Z}$, where $\mathbf{z}_i = \phi_{\text{con}}(\mathbf{h}_i)$
\State Initialize unassigned nodes pool: $\mathcal{U} \gets \{1, \dots, N\}$
\State Initialize final seed set: $\mathcal{S} \gets \emptyset$
\State Compute pairwise cosine similarity matrix $\mathbf{G}$ in contrastive space, where $G_{i, j} = \frac{\mathbf{z}_i^\top \mathbf{z}_j}{\|\mathbf{z}_i\|_2 \|\mathbf{z}_j\|_2}$

\For{$k = 1$ \textbf{to} $K$}
    \If{$\mathcal{U} = \emptyset$}
        \State \textbf{break} \Comment{Pad remaining sequence to maintain static tensor dimensions}
    \EndIf
    \State Select most confident available seed: $i^* \gets \arg\max_{i \in \mathcal{U}} p_i$
    \State $\mathcal{S} \gets \mathcal{S} \cup \{i^*\}$
    \State Initialize capacity accumulator: $C \gets d_{i^*}$
    \State $\mathcal{U} \gets \mathcal{U} \setminus \{i^*\}$ \Comment{Remove seed from candidate pool}
    
    \State Let $\mathcal{J}$ be the sequence of remaining nodes $j \in \mathcal{U}$, sorted descending by $G_{i^*, j}$
  
    \For{\textbf{each} $j \in \mathcal{J}$}
        \If{$C + d_j \leq Q$}
            \State $C \gets C + d_j$
            \State $\mathcal{U} \gets \mathcal{U} \setminus \{j\}$ \Comment{Suppress node from future consideration}
        \EndIf
    \EndFor
\EndFor

\State \Return $\mathcal{S}$
\end{algorithmic}
\end{algorithm}
\FloatBarrier

%% file: appendices/app_formal_proofs_of_symmetry_abstraction.tex
\section{Formal Proofs of Symmetry Abstraction in Neural CFRS}
\label{app:symmetry}
The symmetric Euclidean CVRP is completely characterized by three distinct structural symmetries that arise naturally from its input and solution spaces. 
First, the input space is invariant to spatial transformations: a Euclidean CVRP instance is defined by geographic coordinates~$(x, y)$, and any isometric transformation (translation, rotation, or reflection) strictly preserves all pairwise Euclidean distances. As the routing objective is to minimize total traveled distance, the underlying solution remains identical, and the problem is invariant to the Euclidean group~$E(2)$. 
Second, the solution space is fundamentally unordered: the optimal solution to a CVRP is not a single sequence but an unordered \emph{set} of independent TSP tours, yielding \emph{inter-route permutation invariance}. 
Third, each tour within that set is an undirected cycle anchored at the depot with no inherent traversal direction, giving \emph{intra-route traversal invariance}. Combining the last two points, a CVRP solution is fully described as an unordered set of edge-disjoint undirected cycles.

Autoregressive (AR) sequence models are intrinsically ill-suited to these symmetries.
Step-by-step decoding artificially imposes a strict order between independent vehicle clusters and an arbitrary directional bias on each route. To compensate, AR models rely on expensive data augmentation~\citep{Kim2022Sym-NCO:Optimization, Kwon2020Pomo:Learning, Luo2024NeuralGeneralization} or multi-start decoding~\citep{Kwon2020Pomo:Learning}. 

In contrast, Neural CFRS resolves all three symmetries under the spatial support setting by construction. We structure the formal proofs of this complete abstraction as follows:

\begin{itemize}[leftmargin=*, label={}]
    \item \textbf{[L1] Isometric Invariance of the Input Space (Lemma~\ref{lem:e2-input}):} Establishes that the symmetric Euclidean CVRP is an $E(2)$-invariant problem at the level of cost and feasibility.
    \item \textbf{[L2] Unordered Edge-Set Formulation (Lemma~\ref{lem:edge-set}):} Establishes the unordered edge-set view of the routing solution space.
    \item \textbf{[P1] Neural CFRS $E(2)$ Invariance (Proposition~\ref{prop:e2-cfrs}):} Shows that Neural CFRS inherits $E(2)$ invariance because $\Phi_{\mathrm{SMAE}}$ factors through $E(2)$-invariant features and the downstream \ac{TSP} solver is itself $E(2)$ invariant.
    \item \textbf{[P2] Inter-Route Permutation Invariance (Proposition~\ref{prop:perm}):} Establishes permutation invariance through the equivariance of upstream layers and the invariance of the set-theoretic union over \ac{TSP} tours.
    \item \textbf{[P3] Intra-Route Traversal Invariance (Proposition~\ref{prop:intraroute}):} Establishes intra-route traversal invariance directly from the underlying edge-set formulation.
\end{itemize}

\subsection*{Definitions and Pipeline Formalization}

Let the fixed spatial support $\mathcal{X}$ be defined by its geographic coordinates $X_{\mathrm{support}} \in \mathbb{R}^{(M+1) \times 2}$ (one depot plus $M$ potential customer locations). 
The \acf{SMAE} is a map over this entire support:
\[
    \Phi_{\mathrm{SMAE}} : \mathbb{R}^{(M+1) \times 2} \longrightarrow \mathbb{R}^{(M+1) \times d_{\mathrm{model}}}
\]
By architectural design, $\Phi_{\mathrm{SMAE}}$ never directly observes absolute coordinates. 
Instead, it factors completely through two $E(2)$-invariant intermediates computed over the full city: the $k$-NN adjacency $\mathcal{A}_k(X_{\mathrm{support}})$ and the pairwise Euclidean distance matrix $D(X_{\mathrm{support}})$. 
Concretely, there exists a function $\widetilde{\Phi}$ such that the learned node embeddings with global context $\Psi_{\mathrm{support}}$ is computed as:
\begin{equation}
    \Psi_{\mathrm{support}} \;=\; \Phi_{\mathrm{SMAE}}(X_{\mathrm{support}}) \;=\; \widetilde{\Phi}\!\left(\mathcal{A}_k(X_{\mathrm{support}}),\, D(X_{\mathrm{support}})\right)
    \label{eq:smae-factor}
\end{equation}
A daily \ac{CVRP} instance consists of the depot and $N$ active customers. 
Let $\mathcal{I} \subset \{1, \dots, M\}$ of size $N$ denote the indices of these active locations within the support, with index $0$ designating the depot.
The instance is defined by its strictly positive demand vector $d \in \mathbb{R}_{>0}^N$ and its local coordinates $X_{\mathrm{coord}} = X_{\mathrm{support}}[\{0\} \cup \mathcal{I}]$. 
The neural solver does not re-encode this graph; it simply gathers the corresponding pre-trained embeddings:
\[
    \Psi \;=\; \Psi_{\mathrm{support}}[\{0\} \cup \mathcal{I}] \in \mathbb{R}^{(N+1) \times d_{\mathrm{model}}}
\]

This is the formal content of the architectural choice, stated in
Section \ref{sec:inputs}, that no absolute coordinates are consumed and no positional
encodings are added.
\begin{remark}
While we define $\Phi_{\mathrm{SMAE}}$ as a map from the spatial support coordinates $X_{\mathrm{support}}$, it is important to emphasize that it never consumes these raw coordinates directly. Instead, they are used strictly to compute the $E(2)$-invariant $k$-NN adjacency matrix and the pairwise distance targets for the reconstruction loss. Crucially, once pre-training is complete, the neural network $\Phi_{\mathrm{SMAE}}$ is entirely discarded. For any downstream daily routing problem, no re-encoding of the spatial embeddings is required; the pipeline simply extracts the relevant frozen node embeddings $\Psi_{\mathrm{support}}$ via a fast $O(1)$ table lookup. One can also view $\Psi_{\mathrm{support}}$ as a learned global spatial vocabulary, analogous to word embeddings that are pre-trained.
\end{remark}

The downstream Neural CFRS solver is a deterministic mapping
\[
    f \;:\; (\Psi, d, X_{\mathrm{coord}}) \longmapsto \mathcal{T},
\]
where $\mathcal{T} = \{T_1, \dots, T_K\}$ is the unordered set of $K$ undirected
tours returned by the pipeline. We decompose $f$ into four operators, written
with their full argument lists to avoid ambiguity:
\begin{enumerate}
    \item \textbf{Seed Generation:}
          $h_{\mathrm{seed}}(\Psi, d) \;\mapsto\; C \in \mathbb{R}^{K \times d_{\mathrm{model}}}$,
          extracting $K$ dynamic seed representations.
    \item \textbf{Global Assignment:}
          $h_{\mathrm{assign}}(\Psi, d, C) \;\mapsto\; (\Delta, \widehat{Y})$,
          where the \acf{CT} and entropic \acf{OT}
          layer output a latent distance matrix
          $\Delta \in \mathbb{R}^{N \times K}$ and a continuous transport plan
          $\widehat{Y} \in [0,1]^{N \times K}$.
    \item \textbf{OR Decoding:}
          $h_{\mathrm{CAP}}(\Delta, \widehat{Y}, d, Q) \;\mapsto\; \overline{Y}$,
          where the MIP solver enforces uniform fleet capacity $Q$ to recover
          the discrete boolean assignment matrix
          $\overline{Y} \in \{0,1\}^{N \times K}$.
    \item \textbf{TSP Recovery:}
          $h_{\mathrm{TSP}}(\overline{Y}, X_{\mathrm{coord}})
          \;\mapsto\; \mathcal{T}$,
          mapping each column of $\overline{Y}$ to an independent TSP tour over
          the corresponding customer subset, computed under Euclidean
          distances derived from $X_{\mathrm{coord}}$.\footnote{We pass
          $X_{\mathrm{coord}}$ explicitly to $h_{\mathrm{TSP}}$ to make the
          dependence on Euclidean distances visible. Any standard Euclidean
          TSP solver depends on $X_{\mathrm{coord}}$ only through the
          pairwise distance matrix $D(X_{\mathrm{coord}})$ and is therefore
          itself $E(2)$-invariant.}
\end{enumerate}
The full pipeline is the composition
\begin{align}
    C            &= h_{\mathrm{seed}}(\Psi, d), \nonumber \\
    (\Delta, \widehat{Y}) &= h_{\mathrm{assign}}(\Psi, d, C), \nonumber \\
    \overline{Y} &= h_{\mathrm{CAP}}(\Delta, \widehat{Y}, d, Q), \nonumber \\
    \mathcal{T}  &= h_{\mathrm{TSP}}(\overline{Y}, X_{\mathrm{coord}}).
    \label{eq:pipeline}
\end{align}

\begin{lemma}[Isometric Invariance of the Symmetric Euclidean Input Space]
\label{lem:e2-input}
Let $\mathcal{R}$ denote the set of capacity-feasible routing solutions for
demand vector $d$ and fleet capacity $Q$. The total cost of a routing
solution $R \in \mathcal{R}$ under coordinates $X_{\mathrm{coord}}$ is
\[
    \mathcal{C}(R \mid X_{\mathrm{coord}})
    \;=\; \sum_{(i,j) \in R} \lVert x_i - x_j \rVert_2.
\]
For any isometry $g \in E(2)$ acting on $\mathbb{R}^2$, the set of optimal
routing solutions is preserved:
\[
    \arg\min_{R \in \mathcal{R}} \mathcal{C}(R \mid g \cdot X_{\mathrm{coord}})
    \;=\;
    \arg\min_{R \in \mathcal{R}} \mathcal{C}(R \mid X_{\mathrm{coord}}).
\]
\end{lemma}

\begin{proof}
By the definition of an isometry,
$\lVert g(x_i) - g(x_j) \rVert_2 = \lVert x_i - x_j \rVert_2$ for all $i, j$.
The cost of any route $R$ evaluated under transformed coordinates therefore
satisfies
\[
    \mathcal{C}(R \mid g \cdot X_{\mathrm{coord}})
    \;=\; \sum_{(i,j) \in R} \lVert g(x_i) - g(x_j) \rVert_2
    \;=\; \sum_{(i,j) \in R} \lVert x_i - x_j \rVert_2
    \;=\; \mathcal{C}(R \mid X_{\mathrm{coord}}).
\]
The transformation $g$ acts only on coordinates and alters neither customer
demands nor vehicle capacities, so the feasible set $\mathcal{R}$ is
unchanged. Both the search space and the cost mapping are invariant, hence
the $\arg\min$ sets coincide.
\end{proof}

\begin{lemma}[Unordered Edge-Set Formulation of the Solution Space]
\label{lem:edge-set}
Any complete CVRP solution $\mathcal{T}$ is fully characterized by an
unordered set of pairwise edge-disjoint undirected cycle edge sets.
\end{lemma}

\begin{proof}
A valid CVRP solution assigns customers to $K$ vehicles in disjoint subsets
$V_1, \dots, V_K$, each served by a tour that starts and ends at the depot.
Tour $T_j$ traverses an ordered sequence
$(v_0^{(j)}, v_1^{(j)}, \dots, v_{n_j}^{(j)})$ with
$v_0^{(j)} = v_{n_j}^{(j)} = x_0$ (the depot) and
$\{v_1^{(j)}, \dots, v_{n_j-1}^{(j)}\} = V_j$. Define its undirected edge
set
\begin{equation}
    E_j \;\coloneqq\;
    \bigl\{\{v_{k}^{(j)}, v_{k+1}^{(j)}\}
           \;\big|\; k = 0, 1, \dots, n_j - 1\bigr\}.
    \label{eq:edge-set}
\end{equation}
Because the customer subsets $V_j$ are disjoint, no non-depot edge appears in
two different $E_j$, and depot edges connect the depot to distinct customers
in distinct routes; hence the edge sets $E_1, \dots, E_K$ are pairwise
disjoint. The global solution
$\mathcal{T} = \{E_1, E_2, \dots, E_K\}$ is therefore an unordered set of
pairwise edge-disjoint undirected edge sets. Because each $E_j$ is itself a
set, it has no distinguished starting element and no traversal direction, and
because $\mathcal{T}$ is itself a set, it imposes no ordering on the routes.
Any vehicle relabeling or cycle reversal yields the same mathematical object
$\mathcal{T}$.
\end{proof}
\begin{remark}
    While Lemma~\ref{lem:e2-input} and Lemma~\ref{lem:edge-set} define routing solutions as sets of undirected edges to preserve traversal invariance, we note a specific mathematical edge case regarding single-customer routes. For a route visiting only one customer $v_1$, the traversal sequence $(x_0, v_1, x_0)$ yields two identical undirected edges connected to the depot. To prevent these duplicate edges from collapsing under standard set operations—which would erroneously halve the true Euclidean cost calculated in Lemma 1—the edge sets $E_j$ and the global solution space $\mathcal{T}$ are strictly evaluated as \textbf{multisets}. This ensures exact edge multiplicities are preserved without requiring directed edges, thereby maintaining the intra-route traversal invariance established in the subsequent proofs. For the sake of simplicity and to ensure the core proofs remain intuitive, we maintain the standard terminology of sets.
\end{remark}
\begin{remark}
We note that the global minimum of a \ac{CVRP} instance can technically possess multiple distinct optimal configurations. 
For example, consider a pathological case where 12 customer nodes are distributed evenly around a circle (like a clock face) with the depot at the center. 
If each customer has a demand of 1 and the vehicle capacity is 2, the optimal solution dispatches vehicles to pairs of adjacent nodes.
Due to the exact rotational symmetry, shifting the pairwise assignments by one node yields a completely distinct set of routes with the exact same total Euclidean cost. 
However, such highly symmetric instances are pathological and will almost never occur in standard, randomly sampled \ac{CVRP} benchmarks.
For multiple distinct solutions to yield the exact same objective value, their respective sums of continuous pairwise distances must perfectly equate.
When customer coordinates are drawn independently from a distribution with a density on $\mathbb{R}^2$ (e.g., uniformly over the unit square), the probability of sampling an instance with such perfectly balanced distance combinations is measure zero.
For the returned-solution invariance claims below, we assume that the CAP with learned costs $\Delta$ has a unique optimum and that $h_{\mathrm{CAP}}$ returns it.
\end{remark}

\begin{remark}
Optimality is sufficient under the uniqueness assumption above, but it is not necessary for invariance.
The same results hold for a deterministic CAP decoder that respects seed relabeling, followed by deterministic route recovery based on each clusters' pairwise distances.
A standard solver may be deterministic for identical ordered inputs while returning a different incumbent after variable reordering, since its branching decisions can change.
Wall-clock time limits can also stop otherwise identical runs at different stages of the search.
Our proofs therefore concern decoding operators satisfying the stated assumptions.
In particular, exact decoding of a unique CAP optimum guarantees consistency under seed relabeling, regardless of the solver's search path.
Under these decoding assumptions, Neural \ac{CFRS} preserves $E(2)$ invariance, inter-route permutation invariance, and invariance to intra-route traversal reversal by construction.
This resolves the discussion for inference.
\end{remark}

\begin{remark}
Indeed, more importantly, the symmetry abstraction used during training is independent of discrete CAP decoding.
The model is supervised against reference assignments; the discrete CAP is solved only during inference.
For the spatial-embedding model, the input features remain unchanged under $E(2)$ transformations, while seed relabeling permutes the predicted and reference assignment columns together, leaving the assignment losses unchanged.
Reversing a reference tour likewise leaves its customer group unchanged.
These properties remove the need for data augmentation to learn the stated symmetries, independently of whether downstream CAP decoding returns an optimal solution.
\end{remark}

\begin{proposition}[Neural CFRS $E(2)$ Invariance]
\label{prop:e2-cfrs}
Assume deterministic decoding and that $h_{\mathrm{TSP}}$ depends only on each customer set and its pairwise distance matrix.
For any $g \in E(2)$ acting on $X_{\mathrm{coord}}$, the output
$\mathcal{T}$ of the Neural CFRS pipeline~\eqref{eq:pipeline} is unchanged.
\end{proposition}

\begin{proof}
Both the $k$-NN adjacency $\mathcal{A}_k$ and the pairwise distance matrix $D$ are intrinsically $E(2)$-invariant. 
For every isometry $g \in E(2)$ acting on the full spatial support $X_{\mathrm{support}}$, we have:
\[
\mathcal{A}_k(g \cdot X_{\mathrm{support}}) = \mathcal{A}_k(X_{\mathrm{support}}), \qquad D(g \cdot X_{\mathrm{support}}) = D(X_{\mathrm{support}})
\]
Combined with the factorization in Equation~\eqref{eq:smae-factor}, the global spatial vocabulary is strictly invariant:
\[
\Phi_{\mathrm{SMAE}}(g \cdot X_{\mathrm{support}}) \;=\; \widetilde{\Phi}\!\left(\mathcal{A}_k(X_{\mathrm{support}}), D(X_{\mathrm{support}})\right) \;=\; \Psi_{\mathrm{support}}
\]
Because the global vocabulary $\Psi_{\mathrm{support}}$ does not change under the isometry $g$, the gathered embeddings for any specific daily instance, $\Psi = \Psi_{\mathrm{support}}[\mathcal{I}]$, are consequently unchanged.
The distance-based attention masks used by $h_{\mathrm{seed}}$ and $h_{\mathrm{assign}}$, left implicit in the pipeline notation, are also unchanged under $g$.
Their inputs $\Psi$, $d$, and the attention masks therefore remain unchanged, so the outputs $C$, $\Delta$, and $\widehat{Y}$ are unchanged.
Consequently, $h_{\mathrm{CAP}}$ receives identical inputs and returns the same assignment $\overline{Y}$.
Finally, the customer sets and their pairwise distances are unchanged, so the assumed deterministic route-recovery procedure returns the same unordered collection of routes $\mathcal{T}$.
Therefore, the pipeline:
\[
    f(\Psi, d, X_{\mathrm{coord}}) \;=\; \mathcal{T}
\]
has all inputs that are $E(2)$-invariant and thus the output $\mathcal{T}$ is unchanged for any $g$.
\end{proof}

\begin{proposition}[Neural CFRS Inter-Route Permutation Invariance]
\label{prop:perm}
Assume that $h_{\mathrm{CAP}}$ returns the unique optimum of the exact CAP with learned costs $\Delta$, and that $h_{\mathrm{TSP}}$ applies the same deterministic procedure to each customer set, independently of vehicle indices.
For any permutation $\sigma \in S_K$ with permutation matrix
$P_\sigma \in \{0,1\}^{K \times K}$, applying $P_\sigma^\top$ to the
columns of the discrete assignment matrix $\overline{Y}$ does not change the
output of the TSP recovery operator:
\[
    h_{\mathrm{TSP}}(\overline{Y} P_\sigma^\top, X_{\mathrm{coord}})
    \;=\; h_{\mathrm{TSP}}(\overline{Y}, X_{\mathrm{coord}}).
\]
Consequently, the full pipeline $f$ is invariant to any reindexing of the $K$
seeds (and hence of the $K$ vehicles).
\end{proposition}

\begin{proof}
The TSP recovery operator processes each column of its input matrix
independently and aggregates the resulting tours through the set-theoretic
union:
\[
    h_{\mathrm{TSP}}(\overline{Y}, X_{\mathrm{coord}})
    \;=\; \bigcup_{j=1}^{K}
          \mathrm{TSP}\!\left(\overline{Y}_{\cdot, j},\, X_{\mathrm{coord}}\right).
\]
With our convention for $P_\sigma$,
$(\overline{Y} P_\sigma^\top)_{\cdot, j}
 = \overline{Y}_{\cdot,\, \sigma(j)}$,
so
\[
    h_{\mathrm{TSP}}(\overline{Y} P_\sigma^\top, X_{\mathrm{coord}})
    \;=\; \bigcup_{j=1}^{K}
          \mathrm{TSP}\!\left(\overline{Y}_{\cdot,\sigma(j)},\,
                              X_{\mathrm{coord}}\right)
    \;=\; \bigcup_{j=1}^{K}
          \mathrm{TSP}\!\left(\overline{Y}_{\cdot, j},\,
                              X_{\mathrm{coord}}\right)
    \;=\; h_{\mathrm{TSP}}(\overline{Y}, X_{\mathrm{coord}}),
\]
where the second equality holds because the union over $j$ is unchanged
under the reindexing $j \mapsto \sigma(j)$ of a finite index set.

It remains to argue that any reindexing of the seeds produced by
$h_{\mathrm{seed}}$ manifests precisely as a column permutation of
$\overline{Y}$. Because \ac{CT} attends to the seed
tokens equivariantly with respect to their order, and the log-domain
Sinkhorn--Knopp algorithm is itself permutation-equivariant in its target
marginals~\cite{Cuturi2013SinkhornTransport}, permuting the $K$ seeds by $\sigma$
permutes the columns of both $\Delta$ and $\widehat{Y}$ by~$\sigma$. The MIP
decoder $h_{\mathrm{CAP}}$ then yields $\overline{Y}$ with its columns
permuted in exactly the same way. The display above shows that
$h_{\mathrm{TSP}}$ collapses this column equivariance into invariance, which
is the required statement for $f$.
\end{proof}

\begin{remark}
    While Proposition~\ref{prop:perm} explicitly addresses \emph{seed-index} (column) permutations via the symmetric group $S_K$, the exact same architectural mechanism seamlessly resolves \emph{customer node} (row) permutations via $S_N$. \ac{CT} processes the combined sequence of customer and seed nodes to output a refined representation matrix of size $(N+K+1) \times d_{\text{model}}$. Because the self-attention mechanism is strictly permutation equivariant, any reordering of the input customers or seeds applies the exact same permutation to these output embeddings. Crucially, the latent cost matrix $\Delta$ is computed via pairwise Euclidean distances between these refined customer and seed embeddings ($\Delta_{ij} = \|\overline{h}_i - \overline{c}_j\|_2$). 
    Consequently, any permutation of the input customer nodes propagates equivariantly to the rows of $\Delta$ and the assignment matrix $\overline{Y}$. Conversely, because the intermediate seed generation (Algorithm~\ref{alg:cagd}) extracts anchors deterministically based on value rather than sequence index, the columns of $\Delta$ remain strictly invariant to input permutations. 
    Since the \ac{TSP} recovery operator ultimately evaluates these discrete clusters as unordered sets of vertices, the pipeline mathematically collapses both row equivariance and column invariance into strict global permutation invariance.
\end{remark}

\begin{proposition}[Neural CFRS Intra-Route Traversal Invariance]
\label{prop:intraroute}
The pipeline $f$ is invariant to intra-route traversal direction (e.g.,
clockwise vs.\ counter-clockwise generation).
\end{proposition}

\begin{proof}
For each cluster defined by column $\overline{Y}_{\cdot, j}$ with vertex
set $V_j \cup \{x_0\}$, $h_{\mathrm{TSP}}$ returns an undirected
Hamiltonian cycle on $V_j \cup \{x_0\}$. By Lemma~\ref{lem:edge-set} this
cycle is encoded as an undirected edge set $E_j$ of the form given in
Equation~\eqref{eq:edge-set}, which is by construction invariant under
cyclic rotation and reversal of the underlying vertex sequence.
Aggregating across clusters,
$\mathcal{T} = \{E_1, \dots, E_K\}$ is itself a set, hence intra-route
traversal direction does not affect the output.
\end{proof}

\subsection*{Summary}

Lemma~\ref{lem:e2-input} establishes that the symmetric Euclidean CVRP is an
$E(2)$-invariant problem at the level of cost and feasibility, and
Lemma~\ref{lem:edge-set} establishes the unordered edge-set view of the
solution space. Proposition~\ref{prop:e2-cfrs} shows that Neural CFRS
inherits $E(2)$ invariance because $\Phi_{\mathrm{SMAE}}$ factors through
$E(2)$-invariant features and the only downstream consumer of
$X_{\mathrm{coord}}$ is the TSP solver and the attention masks, which is itself $E(2)$-invariant.
Proposition~\ref{prop:perm} establishes inter-route permutation invariance
through equivariance of the upstream layers and invariance of the union over
TSP tours. Proposition~\ref{prop:intraroute} establishes intra-route
traversal invariance directly from the edge-set formulation. Together these
results give a complete, by-construction abstraction of the three structural
symmetries of the symmetric Euclidean CVRP within the Neural CFRS
framework.

%% file: appendices/app_capacity_aware_surrogate.tex
\section{Capacity-aware Surrogate Loss}
\label{app:cap-aware-surrogate}
We recover assignment probabilities as
$\hat{Y}_{ij}=\Pi^*_{ij}/q_i$.
At convergence, these satisfy
$\sum_j\widehat{Y}_{ij}=1$ and
$\sum_iq_i\widehat{Y}_{ij}\leq 1$.
Observe that this formulation is a demand-weighted surrogate for the discrete CAP objective:
\[
\sum_{i=1}^{N}\sum_{j=1}^{K}\Delta_{ij}\Pi_{ij}
=
\sum_{i=1}^{N}\sum_{j=1}^{K}q_i\Delta_{ij}\hat{Y}_{ij}.
\]
Or alternatively:
\[
\sum_{i=1}^{N}\sum_{j=1}^{K}\frac{\Delta_{ij}}{q_i}\Pi_{ij}
=
\sum_{i=1}^{N}\sum_{j=1}^{K}\Delta_{ij}\hat{Y}_{ij}.
\]
In the first formulation, the costs for the \ac{OT} layer remain as $\Delta_{ij}$ and the costs for the \ac{CAP} problem should $q_i\Delta_{ij}$.
In the second formulation, the costs for the \ac{OT} layer should be $\Delta_{ij}/{q_i}$ and the costs for the \ac{CAP} problem should $\Delta_{ij}$.

We evaluated both choices and observed that the optimality gaps on the validation set for $N=100$ and $N=1000$ was approximately $3.5\%$ and $5.0\%$ respectively towards the end of the training run.
In contrast, using $\Delta_{ij}$ directly in both the \ac{OT} layer and the \ac{CAP} decoder as follows:
\[
\sum_{i=1}^{N}\sum_{j=1}^{K}\Delta_{ij}\Pi_{ij}
\quad\quad\quad
\sum_{i=1}^{N}\sum_{j=1}^{K}\Delta_{ij}{Y}_{ij},
\]
obtained gaps of approximately $3.0\%$ and $4.0\%$ respectively.
We therefore chose the surrogate approach.
This choice retains a demand-weighted transport objective during training and an unweighted assignment objective during decoding.
The \ac{OT} layer therefore serves as a differentiable, capacity-aware supervision mechanism, without requiring its linear objective to coincide with that of the discrete decoder.

We hypothesize that this choice benefits both stages.
Within the \ac{OT} layer, retaining $\Delta_{ij}$ preserves a shared cost scale relative to the entropy parameter; dividing by $q_i$ amplifies costs for low-demand customers and changes the balance between assignment cost and entropy.
Within the \ac{CAP} decoder, each customer contributes its learned assignment cost without an additional demand multiplier, while demand continues to determine capacity consumption.
This is a natural choice for a routing surrogate because travel distance is not explicitly weighted by customer demand.
These considerations offer a possible explanation for the observed performance, although our experiments do not isolate their individual effects.

%% file: appendices/app_training_and_testing_setup.tex
\section{Training and Testing Setup}
\label{app:training-setup}

\paragraph{Hyperparameter Tuning} Guided by \citep{Iyer2023Wide-minimaSchedule}, which posits that sustaining a high \ac{LR} encourages convergence to broad, generalizable minima, we sought to identify the maximum stable \ac{LR} and batch size. 
We conducted a grid search over \acp{LR} $\in [0.001, 0.05]$ and batch sizes $\in \{64, 128, 256, 512\}$. 
Peak validation performance was achieved using the AdamW optimizer \citep{Loshchilov2018FixingAdam} with an \ac{LR} of $0.006$, a batch size of 256 and a dropout rate of $0.1$. All models were trained for 30 epochs on a server equipped with NVIDIA L40S GPUs and 32 core Intel 8562Y+ CPUs.

\paragraph{Model Details}
\begin{enumerate}
    \item \textbf{Phase 1 and 2.} Both \ac{ST} and \ac{CT} are standard Transformer encoders, with 8 heads and 6 layers, a hidden dimension of 128 (feedforward 512), $k_{attn}=20$, $\text{dropout}=0.1$. $\mathcal{L}_\text{con}$ has a scalar weight of $0.05$, and all the other losses are not weighted. We also set $m=3$.
    \item \textbf{Phase 2.} The \ac{OT} layer is configured with $\epsilon=0.01$ and with a maximum iteration of 20. We use $\epsilon=0.001$ and with a maximum iteration of 1000 for inference but found that this did not have a significant impact on performance. The full loss used is
    $\mathcal{L}_{\text{total}} = \big( \mathcal{L}_{\text{seed}} + 0.05\times\mathcal{L}_{\text{con}} \big) + \big(\mathcal{L}_{\text{BCE}} + \mathcal{L}_{\text{OT}}\big)$.
    \item \textbf{Phase 3.} For the \ac{OR} solver, we use Gurobi version $12.0.3$ \citep{GurobiOptimizationLLC2026GurobiManual}, parallelized across 16 CPU cores per batch. To ensure standardized evaluation across all decoding variants, Gurobi is configured with a strict time limit of 100s and a target MIPGap of 0.001. Furthermore, we define the confidence threshold for hard assignment decoding as $\tau_{high} = 0.99$, and the continuous transport plan probability threshold for Sinkhorn-guided sparsification as $\tau_{low} = 10^{-4}$. For the \ac{TSP} solver, we use \ac{LKH-3} \citep{Helsgaun2017AnProblems}.
    \item \textbf{Weight Saving Strategy.} Utilizing the built-in PyTorch \citep{Paszke2019Pytorch:Library} framework for Stochastic Weight Averaging \citep{Athiwaratkun2018ThereAverage}, we initiate the weight averaging process at the midpoint of training (the $50\%$ epoch) and systematically save the model states at each subsequent epoch.
\end{enumerate}

\paragraph{Note on Measuring Inference Time}
To properly contextualize the computational metrics reported throughout our results, we first clarify the composition of our inference pipeline. 
The end-to-end evaluation consists of three distinct stages: (1) the neural network forward pass, (2) the exact Gurobi solver, and (3) parallel TSP recovery. 
Because the first and third stages execute with negligible latency, the total inference time is overwhelmingly dominated by the exact solver, rendering the overall framework fundamentally CPU-bound.

The total inference times reported in our tabular results denote the aggregate parallel wall-clock time required to evaluate an entire dataset at a given scale.
For example, our large-instance evaluation pipeline processes 128 instances across 16 physical CPU cores.
To estimate the average processing time per instance while accounting for 16-way parallelism, we divide the total wall-clock time by $128/16=8$, assuming roughly balanced workloads and negligible parallel overhead.
Figure~\ref{fig:optimality-gap} reports these estimated per-instance processing times.

%% file: appendices/app_ablate_learned_seeds_and_costs.tex
\section{Component Ablations of Seed Selection and Assignment Costs}
\label{app:ablate_learned_seeds_and_costs}

Neural \ac{CFRS} extends the \ac{FJ} framework through learned seed selection and assignment costs.
To assess the contribution of each component, we replace it with its \ac{FJ} counterpart and additionally consider grid-based seeds and customer--seed Euclidean distances as simpler geometric alternatives.
All ablations use the model with spatial embeddings $\Psi$ and the BCE loss $L_{\mathrm{BCE}}$.

\paragraph{Seed Selection.}
Grid-based seeds are obtained by placing regularly spaced anchors across the unit square and mapping each anchor to the nearest customer not already selected.
We compare this geometric construction with \ac{FJ} seed selection and the learned selection procedure of the Seed Transformer.

\paragraph{Assignment Costs.}
The \ac{FJ} cost measures the additional distance incurred by inserting a customer into a depot--seed--depot tour, explicitly accounting for the depot location.
\begin{equation}
    C_{ij}^{\mathrm{FJ}}
    = \|\mathbf{x}_i-\mathbf{x}_0\|_2
    + \|\mathbf{x}_i-\mathbf{x}_{s_j}\|_2
    - \|\mathbf{x}_{s_j}-\mathbf{x}_0\|_2,
\end{equation}
where $\mathbf{x}_0$ is the depot location and $s_j$ is the
customer selected as the seed for vehicle $j$.
We compare these costs with customer--seed Euclidean distances and the latent distances predicted by the Clustering Transformer.

\paragraph{Implementation Details.}
Seeds and costs are replaced at their respective stages while retaining the complete Neural \ac{CFRS} pipeline.
In that sense, the solver budget, parallelism, hardware for these runs are all identical.
In particular, replacement cost matrices are passed through the same \ac{OT} layer before integer assignment.
All variants use the same exact decoding configuration, with no hard assignment or sparsification, followed by LKH-3 route recovery.

%% file: appendices/app_ablations.tex
\section{Ablations}
\label{app:ablations}

\paragraph{Ablations Setup}
To properly contextualize the ablations presented in Table~\ref{tab:ablations}, it is necessary to explicitly distinguish the underlying datasets being evaluated. 
Part of Table~\ref{tab:main_results} reports the "Standard Benchmark Setting" by training and testing the model with the data provided by LEHD \citep{Luo2024NeuralGeneralization}, which generates instances with i.i.d. random coordinates under the standard variable capacity setting of $Q \in \{50, 80, 100, 250\}$ for the respective problem sizes.
Our spatial embeddings can indeed be applied to i.i.d. instances, but that is not the main scope of our comparison in this paper.
Therefore, the "Standard Capacity" ablations in Table 4 strictly evaluate our independently generated fixed spatial support dataset, but subject it to the variable vehicle capacity constraints characteristic of standard benchmarks.
This isolates the architectural impact of $\Psi$ and the assignment loss without confounding the results with spatial distribution shifts.

\paragraph{Apparent Phase Transition for $\Psi$ and $\mathcal{L}_{\text{BCE}}$.} Table \ref{tab:ablations} isolates the impact of the spatial support embeddings ($\Psi$ vs. $(x,y)$) and the assignment loss ($\mathcal{L}_{\text{BCE}}$). We observe an intriguing phase transition: for smaller capacities and graph sizes ($N \le 200$), the spatial prior $\Psi$ anchored with $\mathcal{L}_{\text{BCE}}$ yields the tightest optimality gap (3.03\% for exact). However, under zero-shot generalization to large graphs ($N \ge 500$), discarding $\Psi$ and  $\mathcal{L}_{\text{BCE}}$ loss consistently yields superior generalization (e.g., 3.26\% for hard assignment without $\Psi$ and $\mathcal{L}_{\text{BCE}}$). It is important to note here that this applies both to the constant capacity setting and the standard capacity setting.
We hypothesize that forcing the model to assign nodes to specific seeds causes it to overfit to underlying topological biases, essentially inducing rote memorization. While typically viewed as detrimental, this behavior is actually advantageous in real-world operations where geographic layouts remain static. 
Conversely, without $\Psi$ and $\mathcal{L}_{\text{BCE}}$, the model relies solely on general coordinates and \ac{OT} gradients, prompting it to learn the broader concept of geometry and enforcing capacity constraints rather than relying on spatial memory.
In summary, our ablation studies reveal a clear trade-off: while the spatial support prior ($\Psi$) and assignment loss ($\mathcal{L}_{BCE}$) must be coupled for optimal performance on instances close to the training distribution, utilizing raw coordinates without $\mathcal{L}_{BCE}$ yields superior generalization on larger problems.

\paragraph{Solver Behavior and CAP Subproblem} Under the constant capacity setting for the large-scale $N=1000$ instances detailed in Table~\ref{tab:ablations}, a distinct computational trend emerges regarding the solver execution times. 
With the notable exception of the model utilizing raw coordinates $(x,y)$ combined with the assignment loss $\mathcal{L}_{BCE}$, the average run time consistently increases as the decoding strategy shifts from Exact to Sparsified, and peaks under Hard Assignment, which is counter intuitive as the CAP subproblem in Hard Assignment is actually the smallest.
This progression reveals an underlying mathematical structure within the learned cost matrix $\Delta$. 
Under Exact decoding, the exact solver consistently converges and terminates well before reaching the time limit. 
This early saturation implies that relying solely on the raw values of $\Delta$ limits the achievable optimality gap. 
Under Sparsified decoding, solve times increase, but the resulting performance gains are highly variable and inconclusive, suggesting that merely pruning the edge space does not uniformly simplify the branch-and-bound search.
Conversely, Hard Assignment introduces a counter-intuitive phenomenon: despite aggressively fixing decision variables and strictly reducing the dimensional size of the CAP subproblem (solved exactly), the exact solver runs for a significantly longer duration but ultimately achieves a substantially smaller optimality gap. 
This demonstrates that the neural model confidently and accurately isolates the optimal macro-clusters (i.e. "obviously correct" nodes assigned first), shifting the solver's computational burden entirely to resolving the ambiguous, low-confidence boundary assignments. 
We hypothesize that these unassigned boundary nodes represent the core combinatorial bottleneck of the instance, characterized by a dense set of competing integer-feasible assignments. While forcing the branch-and-bound solver to thoroughly evaluate this dense search space increases computational time, explicitly isolating this bottleneck enables the solver to converge to a better routing solution.

\paragraph{Iterative Fallback.}
The aggressiveness of the Hard Assignment strategy is evidenced by a specific anomaly in the solver's overall execution time. To ensure standardized evaluation, Gurobi is strictly constrained to a 100s time limit and a target MIPGap of 0.001 per instance. Because inference is parallelized across 16 CPU cores for the 128 test samples (yielding 8 sequential batches), the theoretical maximum execution time should inherently bound at approximately 800s. However, the Hard Assignment decoding strategy for the model trained without $\mathcal{L}_{BCE}$ significantly exceeded this threshold, recording an inflated average time of 1240s. 
This inflation occurred because the aggressive node-fixing occasionally rendered the initial residual subproblems intractable within the strict 100s limit, which triggered our iterative fallback mechanism (randomly pruning $10\%$ of the hard-assigned nodes and re-solving).
Notably, this was the only configuration across all runs in Table~\ref{tab:ablations} where the fallback was activated.
This anomaly demonstrates the necessity and effectiveness of the fallback mechanism, which enables a graceful recovery.

\begin{table}[ht]
\centering
\caption{We systematically ablate Neural \ac{CFRS}, confirming that the interplay between $\mathcal{L}_{\text{BCE}}$ and the learned spatial prior $\Psi$ is strictly required to attain peak routing efficiency on smaller instances, yet requires relaxation to maximize zero-shot generalization bounds at larger scales.}
\label{tab:ablations}
\resizebox{\textwidth}{!}{%
\begin{tabular}{c c c l rr rr rr rr}
\toprule
\multirow{2}{*}{\textbf{Setup}} & \multirow{2}{*}{\textbf{Input}} & \multirow{2}{*}{$\mathcal{L}_{\text{BCE}}$} & \multirow{2}{*}{\textbf{Decoding}} & \multicolumn{2}{c}{CVRP100} & \multicolumn{2}{c}{CVRP200} & \multicolumn{2}{c}{CVRP500} & \multicolumn{2}{c}{CVRP1000} \\
\cmidrule(lr){5-6} \cmidrule(lr){7-8} \cmidrule(lr){9-10} \cmidrule(lr){11-12}
& & & & Gap & Time(s) & Gap & Time(s) & Gap & Time(s) & Gap & Time(s) \\
\midrule
\multirow{12}{*}{\rotatebox{90}{\textbf{Constant Capacity}}} & \multirow{6}{*}{$\Psi$} & \multirow{3}{*}{$\checkmark$} & Exact & \textbf{3.03}\% & 77.60 & 3.88\% & 5.41 & 4.54\% & 24.83 & 7.63\% & 359.49 \\
 & & & Sparsified & 3.24\% & 73.44 & \textbf{3.79}\% & 5.12 & 4.46\% & 35.21 & 9.56\% & 410.24 \\
 & & & Hard Assign. & 3.76\% & 75.30 & 4.87\% & 3.92 & 4.79\% & 49.27 & 3.77\% & 609.85 \\
\cmidrule(lr){3-12}
 & & \multirow{3}{*}{$\times$} & Exact & 3.16\% & 81.97 & 3.90\% & 5.65 & 4.38\% & 43.91 & 4.77\% & 172.39 \\
 & & & Sparsified & 3.62\% & 73.04 & 4.96\% & 4.45 & 4.43\% & 37.73 & 13.08\% & 501.21 \\
 & & & Hard Assign. & 4.53\% & 68.87 & 4.95\% & 4.02 & 4.50\% & 50.05 & 4.05\% & 745.55 \\
\cmidrule(lr){2-12}
 & \multirow{6}{*}{$(x, y)$} & \multirow{3}{*}{$\checkmark$} & Exact & 3.34\% & 80.77 & 4.21\% & 5.47 & 5.20\% & 19.62 & 8.37\% & 380.22 \\
 & & & Sparsified & 3.42\% & 76.94 & 4.18\% & 5.04 & 5.22\% & 29.99 & 6.77\% & 397.26 \\
 & & & Hard Assign. & 4.05\% & 74.37 & 4.74\% & 4.58 & 4.87\% & 107.01 & 4.28\% & 381.34 \\
\cmidrule(lr){3-12}
 & & \multirow{3}{*}{$\times$} & Exact & 3.28\% & 102.03 & 3.86\% & 6.75 & \textbf{4.22}\% & 41.38 & 4.07\% & 266.21 \\
 & & & Sparsified & 3.51\% & 96.45 & 3.86\% & 7.05 & 5.58\% & 28.95 & 10.50\% & 481.56 \\
 & & & Hard Assign. & 4.64\% & 89.76 & 5.22\% & 5.68 & 4.27\% & 108.49 & \textbf{3.26}\% & 1240.40 \\
\midrule
\multirow{12}{*}{\rotatebox{90}{\textbf{Standard Capacity}}} & \multirow{6}{*}{$\Psi$} & \multirow{3}{*}{$\checkmark$} & Exact & \textbf{3.03}\% & 87.69 & 4.83\% & 4.41 & 9.74\% & 6.92 & 26.38\% & 41.47 \\
 & & & Sparsified & 3.30\% & 76.35 & 4.75\% & 3.73 & 9.74\% & 9.63 & 28.24\% & 40.00 \\
 & & & Hard Assign. & 3.76\% & 74.22 & 5.06\% & 3.58 & 8.27\% & 16.75 & 22.30\% & 37.90 \\
\cmidrule(lr){3-12}
 & & \multirow{3}{*}{$\times$} & Exact & 3.17\% & 83.48 & 4.80\% & 4.11 & 9.05\% & 7.50 & 27.15\% & 20.81 \\
 & & & Sparsified & 3.71\% & 76.43 & 4.77\% & 3.95 & 9.02\% & 7.22 & 29.76\% & 21.13 \\
 & & & Hard Assign. & 4.56\% & 72.95 & 5.26\% & 3.77 & 8.01\% & 18.96 & 21.43\% & 19.78 \\
\cmidrule(lr){2-12}
 & \multirow{6}{*}{$(x, y)$} & \multirow{3}{*}{$\checkmark$} & Exact & 3.32\% & 76.27 & 5.14\% & 3.58 & 10.49\% & 6.03 & 25.02\% & 18.71 \\
 & & & Sparsified & 3.41\% & 74.26 & 5.09\% & 3.68 & 10.67\% & 6.58 & 26.30\% & 20.21 \\
 & & & Hard Assign. & 4.05\% & 70.81 & 5.13\% & 3.34 & 9.02\% & 6.49 & 21.31\% & 17.79 \\
\cmidrule(lr){3-12}
 & & \multirow{3}{*}{$\times$} & Exact & 3.28\% & 81.93 & 4.75\% & 3.97 & 8.44\% & 7.06 & 20.90\% & 20.28 \\
 & & & Sparsified & 3.69\% & 77.36 & \textbf{4.69}\% & 3.89 & 8.41\% & 7.85 & 23.74\% & 22.65 \\
 & & & Hard Assign. & 4.63\% & 71.96 & 4.88\% & 3.59 & \textbf{7.46}\% & 8.56 & \textbf{18.95}\% & 19.31 \\
\bottomrule
\end{tabular}
}
\end{table}

%% file: appendices/app_single_layer_model.tex
\section{Isolating Representational Power: Results on Single-Layer Model}
\label{app:one-layer}

In standard \ac{NCO} architectures, deep Transformer encoders (6 layers or more) are deemed necessary to allow nodes to iteratively aggregate spatial information and construct a global geometric understanding of the routing instance.
To isolate the representational power of $\Psi$, our learned embeddings, from the model capacity of Transformer encoders, we evaluate Neural \ac{CFRS} under an extreme bottleneck: restricting both the \ac{ST} and \ac{CT} to a single attention layer. 
This reduces the total model footprint to a mere 682K parameters. As standard capacity settings introduces severe out-of-distribution distribution shifts for 1-layer models, our analysis here focuses exclusively on the constant capacity operational setting ($Q=50$), which most closely mirrors the homogeneous fleet configurations of real-world logistics networks.

\paragraph{The Surprising Performance of Neural CFRS} 
A notable finding from this ablation is the surprising efficacy of the core Neural \ac{CFRS} pipeline, even when raw coordinates and one layer is used. 
As demonstrated in Table \ref{tab:one-layer}, the 1-layer model operating purely on raw $(x,y)$ and without $\mathcal{L}_{\text{BCE}}$ achieves optimality gaps of approximately $7.42\%$ on $N=500$ and $7.31\%$ on $N=1000$ under hard assignment decoding.
In the context of large-scale \ac{NCO}, achieving gaps in the $7\%$ to $8\%$ range on large instances using an ultra-lightweight, \ac{NAR} model is highly competitive.
This indicates that the fundamental \ac{CFRS} formulation—framing the problem as a differentiable \ac{OT} assignment and sparsifying an exact solver—is intrinsically sound.
The network does not strictly require deep computational graphs to yield viable, capacity-compliant clusters.

\paragraph{Elevating Performance via the Spatial Support Prior.} 
While the raw coordinates without $\mathcal{L}_{\text{BCE}}$ baseline demonstrates the robustness of the \ac{OT} formulation, introducing the pre-trained spatial support embeddings ($\Psi$) significantly elevates generalization performance. By leveraging $\Psi$, the 1-layer model reduces the optimality gaps down to $5.58\%$ for $N=500$ and $5.08\%$ for $N=1000$ from $7.42\%$ and $7.31\%$ respectively.
This significant performance delta empirically validates the core hypothesis of the Neural \ac{CFRS} framework: pre-training of \ac{SMAE} successfully offloads the burden of geometric representation learning as the $\Psi$ embeddings capture the global $k$-NN topology and the relative distance to every node in the fixed service area, which enables the downstream routing model to focus on learning the clusters.

\paragraph{Implications for Parameter Efficiency.}
These results challenge the prevailing \ac{NCO} paradigm that deeper sequence models are inherently necessary for large-scale routing. By exploiting a fixed spatial support under a constant capacity regime, Neural \ac{CFRS} proves that injecting an $E(2)$-invariant topological prior enables extreme parameter efficiency. 
This positions the framework as uniquely suited for amortized, latency-sensitive environments, where executing a 682K-parameter, one-shot assignment model over a pre-computed spatial vocabulary offers immense computational advantages without sacrificing routing quality.

In summary, our ablation studies on the single-layer models reveal a clear trade-off: incorporating $\mathcal{L}_{BCE}$ is essential for maximizing performance near the training distribution, whereas omitting it yields superior generalization.

\begin{table}[ht]
\centering
\caption{Ablation study isolating representational power by restricting the architecture to a single \ac{ST} layer and a single \ac{CT} layer. Models leveraging the spatial support prior ($\Psi$) consistently outperform baselines utilizing raw coordinates $(x, y)$. This performance gap empirically confirms that the pre-trained embeddings successfully offload the burden of geometric representation learning, enabling extreme parameter efficiency.}
\label{tab:one-layer}
\resizebox{\textwidth}{!}{%
\begin{tabular}{c c c l rr rr rr rr}
\toprule
\multirow{2}{*}{\textbf{Setup}} & \multirow{2}{*}{\textbf{Input}} & \multirow{2}{*}{$\mathcal{L}_{\text{BCE}}$} & \multirow{2}{*}{\textbf{Decoding}} & \multicolumn{2}{c}{CVRP100} & \multicolumn{2}{c}{CVRP200} & \multicolumn{2}{c}{CVRP500} & \multicolumn{2}{c}{CVRP1000} \\
\cmidrule(lr){5-6} \cmidrule(lr){7-8} \cmidrule(lr){9-10} \cmidrule(lr){11-12}
& & & & Gap & Time(s) & Gap & Time(s) & Gap & Time(s) & Gap & Time(s) \\
\midrule
\multirow{12}{*}{\rotatebox{90}{\textbf{Constant Capacity}}} & \multirow{6}{*}{$\Psi$} & \multirow{3}{*}{$\checkmark$} & Exact & \textbf{4.820}\% & 96.62 & 5.545\% & 6.53 & 6.792\% & 28.22 & 6.492\% & 246.15 \\
 & & & Sparsified & 4.844\% & 92.29 & \textbf{5.395}\% & 6.34 & 6.745\% & 42.99 & 9.769\% & 410.14 \\
 & & & Hard Assign. & 5.681\% & 90.10 & 5.733\% & 6.16 & 5.946\% & 76.00 & 5.315\% & 262.98 \\
\cmidrule(lr){3-12}
 & & \multirow{3}{*}{$\times$} & Exact & 5.215\% & 78.88 & 6.040\% & 5.07 & 6.595\% & 15.27 & 6.517\% & 62.88 \\
 & & & Sparsified & 5.338\% & 76.24 & 5.969\% & 5.18 & 6.995\% & 45.49 & 15.163\% & 535.94 \\
 & & & Hard Assign. & 6.015\% & 73.47 & 6.115\% & 4.73 & \textbf{5.579}\% & 40.33 & \textbf{5.083}\% & 318.25 \\
\cmidrule(lr){2-12}
 & \multirow{6}{*}{$(x, y)$} & \multirow{3}{*}{$\checkmark$} & Exact & 5.372\% & 74.93 & 6.433\% & 4.69 & 9.396\% & 18.63 & 9.610\% & 128.83 \\
 & & & Sparsified & 5.399\% & 75.27 & 6.186\% & 4.67 & 9.243\% & 30.26 & 9.653\% & 204.91 \\
 & & & Hard Assign. & 6.177\% & 74.01 & 6.337\% & 4.90 & 8.154\% & 122.86 & 7.991\% & 133.46 \\
\cmidrule(lr){3-12}
 & & \multirow{3}{*}{$\times$} & Exact & 5.683\% & 77.38 & 6.437\% & 5.72 & 9.008\% & 26.56 & 8.916\% & 127.88 \\
 & & & Sparsified & 5.618\% & 72.76 & 7.431\% & 5.10 & 9.154\% & 29.57 & 17.220\% & 392.29 \\
 & & & Hard Assign. & 6.316\% & 69.91 & 6.565\% & 3.95 & 7.417\% & 60.51 & 7.314\% & 157.19 \\
\midrule
\multirow{12}{*}{\rotatebox{90}{\textbf{Standard Capacity}}} & \multirow{6}{*}{$\Psi$} & \multirow{3}{*}{$\checkmark$} & Exact & 4.845\% & 96.13 & 7.413\% & 5.39 & 12.113\% & 8.16 & 24.953\% & 25.67 \\
 & & & Sparsified & \textbf{4.838}\% & 93.38 & 7.038\% & 5.49 & 12.581\% & 9.00 & 29.701\% & 27.37 \\
 & & & Hard Assign. & 5.682\% & 91.02 & 6.589\% & 5.18 & 9.950\% & 19.99 & 20.525\% & 25.20 \\
\cmidrule(lr){3-12}
 & & \multirow{3}{*}{$\times$} & Exact & 5.205\% & 82.00 & 7.315\% & 4.35 & 12.041\% & 7.31 & 20.845\% & 20.71 \\
 & & & Sparsified & 5.270\% & 75.57 & 7.689\% & 4.03 & 12.511\% & 7.22 & 26.589\% & 21.44 \\
 & & & Hard Assign. & 6.007\% & 73.93 & \textbf{6.477}\% & 3.77 & \textbf{9.352}\% & 9.81 & \textbf{16.728}\% & 18.54 \\
\cmidrule(lr){2-12}
 & \multirow{6}{*}{$(x, y)$} & \multirow{3}{*}{$\checkmark$} & Exact & 5.397\% & 75.47 & 8.239\% & 3.72 & 16.163\% & 5.76 & 31.586\% & 18.75 \\
 & & & Sparsified & 5.366\% & 72.27 & 7.896\% & 3.64 & 16.395\% & 6.92 & 35.809\% & 19.64 \\
 & & & Hard Assign. & 6.207\% & 78.16 & 7.460\% & 4.51 & 12.671\% & 9.18 & 23.193\% & 20.69 \\
\cmidrule(lr){3-12}
 & & \multirow{3}{*}{$\times$} & Exact & 5.684\% & 77.32 & 8.074\% & 4.10 & 16.705\% & 7.28 & 31.542\% & 20.00 \\
 & & & Sparsified & 5.740\% & 72.80 & 7.969\% & 3.73 & 17.038\% & 6.62 & 36.582\% & 20.27 \\
 & & & Hard Assign. & 6.347\% & 70.38 & 7.279\% & 3.53 & 12.730\% & 11.26 & 23.844\% & 17.96 \\
\bottomrule
\end{tabular}
}
\end{table}

%% file: appendices/app_unseen_robustness.tex
\section{Robustness to Unseen Nodes}
\label{app:unseen-robustness}

We evaluate whether Neural \ac{CFRS} can accommodate customer locations outside its pretraining vocabulary through embedding interpolation, without retraining or fine-tuning.

\paragraph{Setup.}
We sample an evaluation pool $\mathcal{X}_{\mathrm{oov}}$ of $10{,}000$ unseen customer coordinates uniformly from the same unit-square region as the original $3{,}000$-location support $\mathcal{X}$.
The two sets are disjoint, and the unseen locations are excluded from pretraining and routing-model training.
The finite evaluation pool is an implementation choice that matches well with our architecture.
Its coordinates are sampled i.i.d.\ from the uniform distribution on $[0,1]^2$, and test instances are generated independently conditional on the fixed spatial supports. 
The same interpolation procedure applies directly to fresh coordinates sampled for each instance, allowing new customer locations to be handled without retraining.

For each problem size $N \in \{100,200,500,1000\}$ and unseen fraction $r \in \{0.05,0.10,0.15,0.20,0.50,1.00\}$, we independently generate instances containing $\lfloor rN \rfloor$ customers sampled without replacement from $\mathcal{X}{\mathrm{oov}}$ and the remaining customers from $\mathcal{X}$.
The depot remains fixed, customer demands are sampled independently and uniformly from ${1,\ldots,9}$, and vehicle capacity is $Q=50$.

\paragraph{Embedding interpolation.}
For each unseen location $u$, we identify its five nearest neighbors $\mathcal{N}_5(u)$ in the original support using Euclidean distance.
We take the inverse distance to each neighbor and normalize the weights to sum to one:
\begin{equation}
w_{ui}
=
\frac{\|x_u-x_i\|_2^{-1}}
{\sum_{j\in\mathcal{N}_5(u)}\|x_u-x_j\|_2^{-1}},
\qquad i\in\mathcal{N}_5(u).
\end{equation}
The interpolated embedding is then
\begin{equation}
\Psi_u
=
\sum_{i\in\mathcal{N}_5(u)}w_{ui}\Psi_i,
\end{equation}
where $\Psi_i$ denotes the pretrained embedding of location $i$.
Known locations retain their original embeddings.
The interpolated embeddings are passed to the existing routing model, with all model parameters fixed.

\paragraph{Evaluation.}
For each $(N,r)$ configuration, we evaluate $1{,}000$ instances at $N=100$ and $128$ instances at each larger size.
We use exact \ac{CAP} decoding with Gurobi, followed by LKH-3 for route recovery.
Table~\ref{tab:unseen-locations} reports mean routing gaps relative to HGS reference solutions.
Across the evaluated settings, gaps range from $2.63\%$ to $4.32\%$.
At $N=1000$, the gap is $4.08\%$ when all customers occupy unseen locations.
These results demonstrate the usefulness of embedding interpolation for accommodating new customer locations within the same geographic region.

\paragraph{Discussion.}
All evaluated unseen-location configurations achieve lower mean routing gaps than the original $0\%$-unseen evaluation under exact \ac{CAP} decoding.
One possible explanation is that averaging nearby embeddings reduces reliance on location-specific patterns learned during training, producing an effect similar to test-time regularization.
These results suggest a possible regularizing effect of embedding interpolation, although we have not isolated this mechanism experimentally.
We therefore present this explanation as a hypothesis.

%% file: appendices/app_comparison_with_sota.tex
\section{Comparison of Neural CFRS with other comparable SOTA Neural Methods}
\label{app:sota-comparison}

\paragraph{Baseline Reproduction Methodology} To ensure a rigorous comparison, baseline results were obtained through a combination of direct citation, execution of official repositories, and custom implementations.
Note that prior work in general report total inference time across all instances and thus we follow that approach.
Performance metrics for LKH-3, HGS, OR-Tools, BQ (greedy), and LEHD (greedy) are directly cited from \citet{Luo2024NeuralGeneralization}. 
For POMO, NCC, and GLOP, we generated results by running inference using the authors' official codebases and provided pre-trained models. 
Specifically, we utilized the POMO model trained on size-100 instances via \ac{RL}, and the NCC model trained on size-200 instances (reproduction details for NCC and GLOP are discussed below). 
Because neither code nor pre-trained weights were publicly available for TAM, we report the results directly from their original publication, omitting the $N=200$ and $N=500$ scales as they were not provided. 
We utilize the Petal algorithm implementation provided by \citet{Rasku2019Meta-surveyResults}, though it should be noted that their solver stalled at $N=1000$.
Finally, Sweep-LKH3 and Fisher-Jaikumar-LKH3 were evaluated using our own custom implementations.

\paragraph{Running Inference on NCC}
We utilized the official implementation provided by the authors, loading the model weights from the checkpoint specified in their evaluation scripts (\texttt{epoch=117\_val\_acc=0.9592.ckpt}). As NCC inference is executed serially, processing times range from approximately 2 seconds per instance for 100-node problems to 30 seconds for 1000-node problems. 
To ensure computational tractability, we limited our evaluation of their method to 128 samples per problem size. 
We anticipate that this reduction does not significantly impact the overall performance trends for CVRP100. 
All hyperparameters follow the \texttt{ckmeans\_greedy} configuration from their codebase.

\paragraph{Running Inference on GLOP} For the GLOP \citep{Ye2024Glop:Real-time} baseline, the original authors provided pretrained weights exclusively for large-scale instances (CVRP1000 and CVRP2000). 
To ensure a comprehensive and fair comparison, we trained a dedicated GLOP partitioner model from scratch on CVRP100 instances. 
In the official implementation, the sparse neighborhood graph size ($\texttt{K\_SPARSE}$) is set to $10\%$ of the total node count (e.g., $k=100$ for $N=1000$). 
To maintain consistency, we adapted the sparsity parameter accordingly, setting $\texttt{K\_SPARSE}=10$ for our CVRP-100 training.

We trained the partitioner for 30 epochs using the authors' prescribed hyperparameters. The entire training process required approximately 4 hours, which exceeded our standard baseline training budget of 3 hours. 
The original code monitors the validation objective throughout the process and we observed that the model achieved its best validation objective at epoch 20; therefore, we utilized this specific checkpoint for all subsequent evaluations.

During inference, we coupled our trained CVRP100 partitioner with the authors' default stage-two neural reviser, which operates on sub-problems of size 20. 
We utilized greedy decoding with 5 revision iterations, consistent with the standard configuration.
To obtain our final results, we evaluated this single model trained on CVRP100 across multiple problem scales ($N=100, 200, 500$, and $1000$) on the LEHD \citep{Luo2024NeuralGeneralization} benchmark datasets. 
We found that maintaining the $10\%$ nearest-neighbor sparsity ratio during inference (e.g., dynamically setting $\texttt{K\_SPARSE}=50$ for $N=500$ and $100$ for $N=1000$) yielded the strongest overall generalization performance for GLOP across all evaluated scales.
Table \ref{tab:sota-comparison} shows that GLOP has poor performance on small problem sizes.

\paragraph{Discrepancies in the LEHD Dataset}
We observe a notable performance discrepancy when evaluating the Sweep heuristic \citep{Gillett1974AProblem}: it yields surprisingly strong results on the dataset provided by \citet{Luo2024NeuralGeneralization}, yet performs poorly on instances where $N \in \{100, 200, 500\}$. To investigate the root cause of this anomaly, we generated 1,000 independent test instances. Unlike the dataset in \citet{Luo2024NeuralGeneralization}, where depots are randomly positioned, our instances enforce a strict spatial support with a centrally located depot. We then established strong near-optimal baselines for these instances using \ac{HGS} \citep{Vidal2022HybridNeighborhood}. To ensure consistent evaluation, all instances were executed in parallel without CPU contention (allocating one physical core per process) and run for the recommended \ac{HGS} time limit of $T = 240 \times (N/100)$ seconds. 

Table \ref{tab:lehd-discrepancy} details the resulting optimality gaps for the Sweep algorithm against these new baselines. While structural differences in the problem geometries might partially explain the variation, we hypothesize that the baseline solutions provided in the original CVRP1000 dataset—which were generated using LKH-3—may simply be highly suboptimal, thereby artificially inflating the apparent performance of the Sweep algorithm. Fully verifying and resolving this discrepancy remains an important direction for future research.

\paragraph{Discussion of Results}
Neural \ac{CFRS} (exact) achieves the most competitive average optimality gap of $2.732\%$. 
Notably, there are 174 individual instances where Neural \ac{CFRS} outperforms the \ac{LKH-3} baseline. 
Figure \ref{fig:optimality_gaps_histogram_lehd} illustrates the distribution of these optimality gaps, while Figure \ref{fig:n100_lehd_best} visualizes the routing configurations for the two best-performing instances.

Furthermore, the Fisher-Jaikumar-LKH3 heuristic demonstrates remarkably strong performance, maintaining optimality gaps of approximately $6\%$ across the CVRP100, 200, and 500 datasets.
Excluding TAM-LKH3, Fisher-Jaikumar-LKH3 emerges as the highest-performing method among the evaluated models for CVRP1000. 
While this relative success might be partially attributed to weaker baseline configurations, it underscores the enduring efficacy of a classical algorithm developed nearly four decades ago.
Ultimately, these results highlight the validity of the clustering-based approach, suggesting that further research into modern, neural \ac{CFRS} methodologies is highly warranted.

\begin{figure}[ht]
    \centering
    \includegraphics[width=0.80\textwidth]{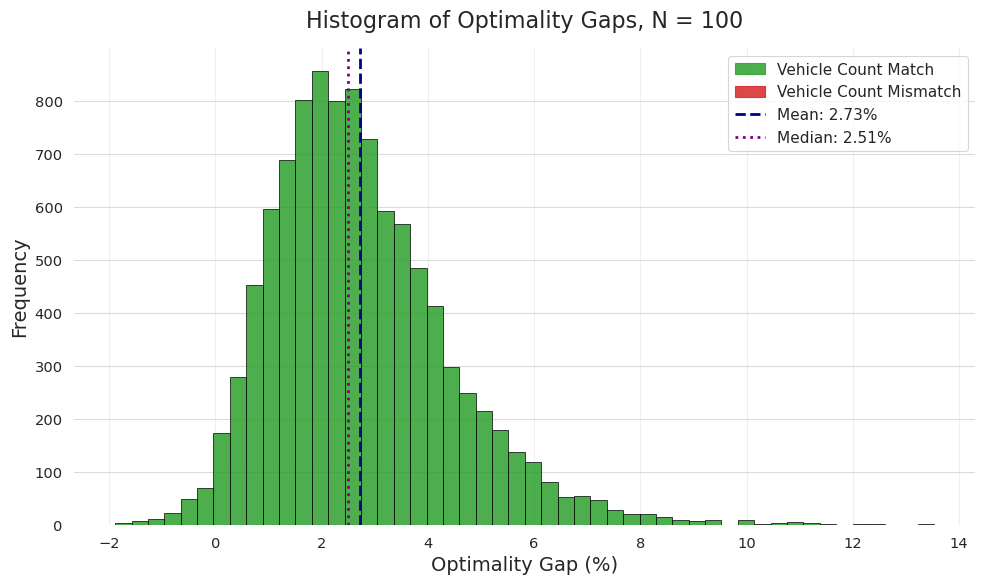}
    \caption{Histogram detailing the distribution of optimality gaps for Neural CFRS (exact) on the $N=100$ dataset. The distribution highlights the 174 instances (left of $0\%$ gap) where the proposed method successfully outperformed the LKH3 baseline. Notably, for all these instances, $K_{\text{min}}$ accurately matches the ground truth $K$.}
    \label{fig:optimality_gaps_histogram_lehd}
\end{figure}

\begin{figure}[ht]
    \centering
    \includegraphics[width=\textwidth]{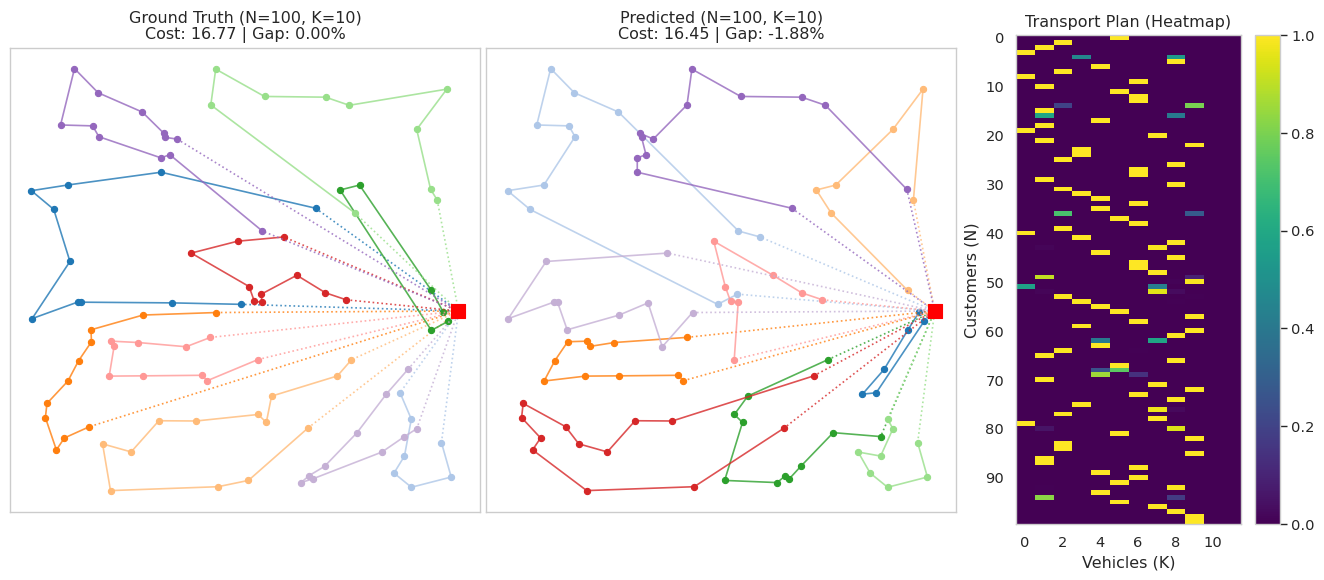}
    \includegraphics[width=\textwidth]{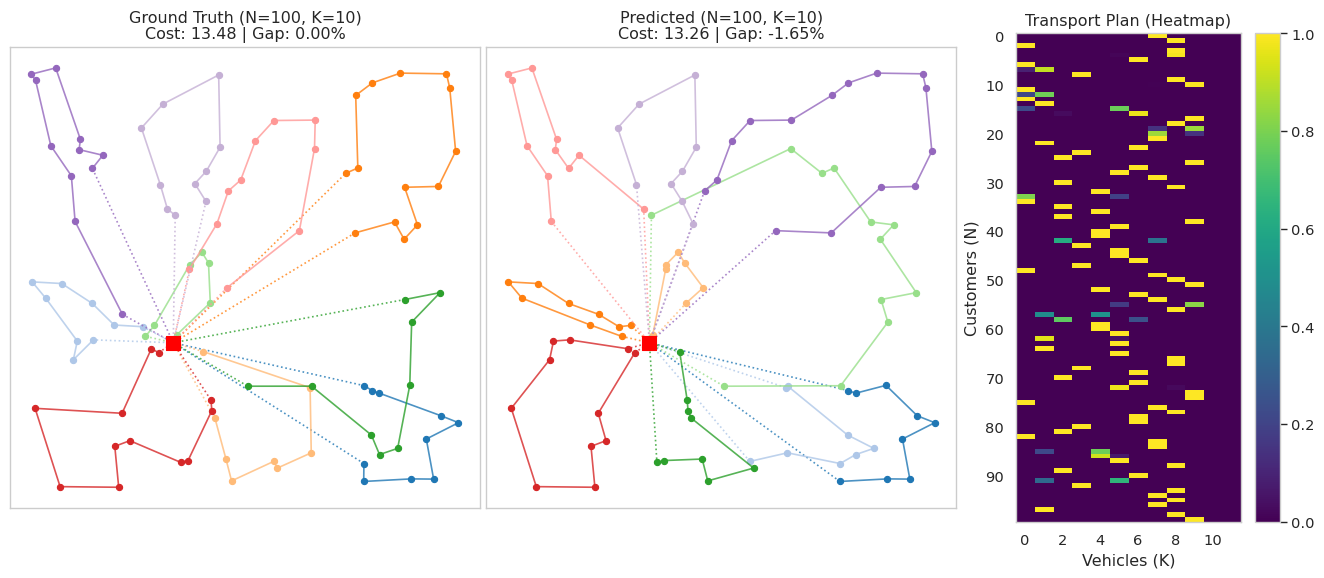}
    \caption{Visual comparison of the baseline \ac{LKH-3} solutions and the Neural \ac{CFRS} (exact) predictions for the two $N=100$ instances that yielded the most significant performance improvements. In these specific cases, Neural \ac{CFRS} outperforms the baseline by $1.88\%$ and $1.65\%$, respectively. The corresponding transport plans show a generally high level of confidence in the assignments, with many of the probabilities close to $1.0$.}
    \label{fig:n100_lehd_best}
\end{figure}

\begin{table}[ht]
\centering
\caption{Performance comparison of various \ac{CVRP} solvers across different problem sizes ($N=100, 200, 500, 1000$). The table reports the average optimality gap (\%) relative to the LKH3 baseline and the total computation time. The evaluated methods are grouped into classical heuristics, neural \ac{AR} approaches, clustering-based (neural and classical) algorithms, and our proposed Neural \ac{CFRS} variants. Bold values indicate the lowest optimality gap among the learning and heuristic-based solvers for each dataset size, while dashes (-) denote instances where a solver failed to complete or the data was unavailable.}
\label{tab:sota-comparison}
\resizebox{\textwidth}{!}{ 
\begin{tabular}{ll rr rr rr rr}
\toprule
\multicolumn{2}{c}{\multirow{2}{*}{\textbf{Method}}} & \multicolumn{2}{c}{CVRP100} & \multicolumn{2}{c}{CVRP200} & \multicolumn{2}{c}{CVRP500} & \multicolumn{2}{c}{CVRP1000} \\
\cmidrule(lr){3-4} \cmidrule(lr){5-6} \cmidrule(lr){7-8} \cmidrule(lr){9-10}
\multicolumn{2}{c}{} & Gap & Time & Gap & Time & Gap & Time & Gap & Time \\
\midrule
\multicolumn{2}{l}{LKH3} & 0.000\% & 12h & 0.000\% & 2.1h & 0.000\% & 5.5h & 0.000\% & 7.1h \\
\multicolumn{2}{l}{HGS} & -0.533\% & 4.5h & -1.126\% & 1.4h & -1.794\% & 4h & -2.162\% & 5.3h \\
\multicolumn{2}{l}{OR-Tools} & 6.193\% & 2h & 6.894\% & 1h & 9.112\% & 2.2h & 11.662\% & 3h \\
\midrule
\multicolumn{2}{l}{POMO greedy} & 3.162\% & 4.7s & 9.335\% & 0.2s & 34.685\% & 0.5s & 307.928\% & 1.4s \\
\multicolumn{2}{l}{BQ greedy} & 2.993\% & 0.7m & 3.527\% & 4s & 5.121\% & 0.4m & 9.812\% & 2.4m \\
\multicolumn{2}{l}{LEHD greedy} & 3.648\% & 0.5m & \textbf{3.312\%} & 3s & \textbf{3.178\%} & 0.3m & 4.912\% & 1.6m \\
\midrule
\multicolumn{2}{l}{Sweep-LKH3} & 16.097\% & 2.5m & 13.095\% & 2s & 14.150\% & 3.8s & 5.230\% & 8.8s \\
\multicolumn{2}{l}{Petal} & 8.647\% & 31m & 8.214\% & 2.4m & 11.113\% & 26m & - & - \\
\multicolumn{2}{l}{Fisher-Jaikumar-LKH3} & 6.750\% & 18m & 6.454\% & 25s & 6.740\% & 2.8m & \textbf{3.749\%} & 1.3m \\
\multicolumn{2}{l}{Neural Cap. Cluster$^*$} & 10.088\% & 5.2m & 11.545\% & 11m & 11.267\% & 31m & 20.025\% & 1.1h \\
\multicolumn{2}{l}{TAM-LKH3} & 3.209\% & 0.86s & - & - & - & - & \textbf{-0.215\%} & 6.2s \\
\multicolumn{2}{l}{GLOP} & 22.70\% & 16m & 18.78\% & 14s & 17.57\% & 19s & 31.10\% & 27s \\
\midrule
\multirow{3}{*}{Neural \ac{CFRS}} & Exact & \textbf{2.732\%} & 1.4m & 6.247\% & 4.0s & 9.787\% & 6.8s & 30.102\% & 21s \\
 & Sparsified & 3.753\% & 1.2m & 5.964\% & 3.9s & 9.604\% & 7.7s & 31.219\% & 21s \\
 & Hard Assign. & 3.858\% & 1.2m & 5.603\% & 4.3s & 7.827\% & 8.4s & 27.868\% & 20s \\
\bottomrule
\multicolumn{9}{l}{\footnotesize * represents 128 instances used for all problem sizes. Dashes indicate instances where a baseline failed to complete or data was unavailable.}
\end{tabular}
}
\end{table}

\begin{table}[ht]
  \centering
  \caption{Performance and inference time (in seconds) of the Sweep algorithm evaluated on our independently generated instances with centrally located depots, solved to near-optimality with \ac{HGS} with the recommended time limit of $T = 240 \times (N/100)$ seconds. The consistently large optimality gaps observed across all scales suggest that the strong Sweep performance on the original LEHD CVRP1000 dataset may be an artifact of weak baseline solutions.}
  \label{tab:lehd-discrepancy}
  \begin{tabular}{l cc @{\hspace{1.5em}} cc @{\hspace{1.5em}} cc @{\hspace{1.5em}} cc}
    \toprule
    & \multicolumn{2}{c}{CVRP100} & \multicolumn{2}{c}{CVRP200} & \multicolumn{2}{c}{CVRP500} & \multicolumn{2}{c}{CVRP1000} \\
    \midrule
    Sweep & 10.873\% & 15s & 14.119\% & 27s & 20.592\% & 56s & 24.936\% & 54s \\
    \bottomrule
  \end{tabular}
\end{table}

%% file: appendices/app_one_extra_bin.tex
\section{How often is the optimal bin size underestimated?}
\label{app:k_plus_one}

In the \ac{CVRP}, the theoretical lower bound on the required fleet size is defined as $K_{\min} = \lceil \sum_{i=1}^N d_i / Q \rceil$. Under our experimental setup, node demands are uniformly sampled as $d_i \sim \mathcal{U}(1, 9)$ with a fixed vehicle capacity of $Q = 50$. Because the maximum demand ($d_{\max} = 9$) constitutes less than $20\%$ of the total vehicle capacity, and the uniform distribution ensures an abundance of small items to fill capacity gaps, the $K_{\min}$ bound is typically tight. Indeed, for large $N$, the probability of encountering a routing instance that strictly requires more than $K_{\min}$ vehicles is practically negligible. Consequently, we observe that computing $K_{\min}$ in this manner consistently yields a feasible fleet size across our evaluations.

Nevertheless, due to the strict constraints of the underlying bin packing problem, highly specific discrete demand distributions can theoretically introduce packing inefficiencies that necessitate an additional vehicle (i.e., $K = K_{\min} + 1$). To evaluate the robustness of our framework against these capacity boundaries, we analyze instances of size $N \in \{100, 200\}$ under the constant capacity setting, utilizing our best-performing model configuration (employing $\Psi$ and the $\mathcal{L}_{\text{BCE}}$ objective). 

As illustrated in Figure \ref{fig:optimality_gaps_histogram}, which details the distribution of optimality gaps for $N \in \{100, 200\}$, the vast majority of our predicted solutions match the near-optimal \ac{HGS} solutions in their required number of vehicles ($K$). For the minority of instances that fail to match, we find that supplying an extra vehicle allows the model's fleet size to align with the \ac{HGS} solutions. This suggests a straightforward heuristic to improve overall performance: execute the routing procedure for both $K_{\min}$ and $K_{\min} + 1$, and select the superior solution. Table \ref{tab:k_comparison} details the full comparison under this setup, demonstrating a strict, albeit variable, improvement in performance. For instance, using our best-performing configuration ($\Psi$ and $\mathcal{L}_\text{BCE}$), the optimality gap improves by a marginal $0.02\%$ (from $3.03\%$ to $3.01\%$) for CVRP100, but yields a substantial $3.32\%$ improvement (from $7.63\%$ to $4.31\%$) for CVRP1000.

Finally, to better understand the limitations of our method, we perform a deeper analysis of the worst-performing instances for $N \in \{100, 200\}$. Figures \ref{fig:n100_worst_different_k} and \ref{fig:n200_worst_different_k} indicate that a primary driver of poor performance is simply a discrepancy in the predicted $K$. Furthermore, Figure \ref{fig:n100_worst_short_routes} highlights another potential failure mode: the ground truth containing short routes near the depot.
This suggests that the model has a strong bias toward maximizing vehicle capacity utilization (greedy packing), struggling to recognize edge cases where deploying an under-utilized vehicle is optimal.
Improving the model's ability to balance capacity utilization with spatial efficiency presents a compelling avenue for future research.

\begin{figure}[ht]
    \centering
    \includegraphics[width=0.48\textwidth]{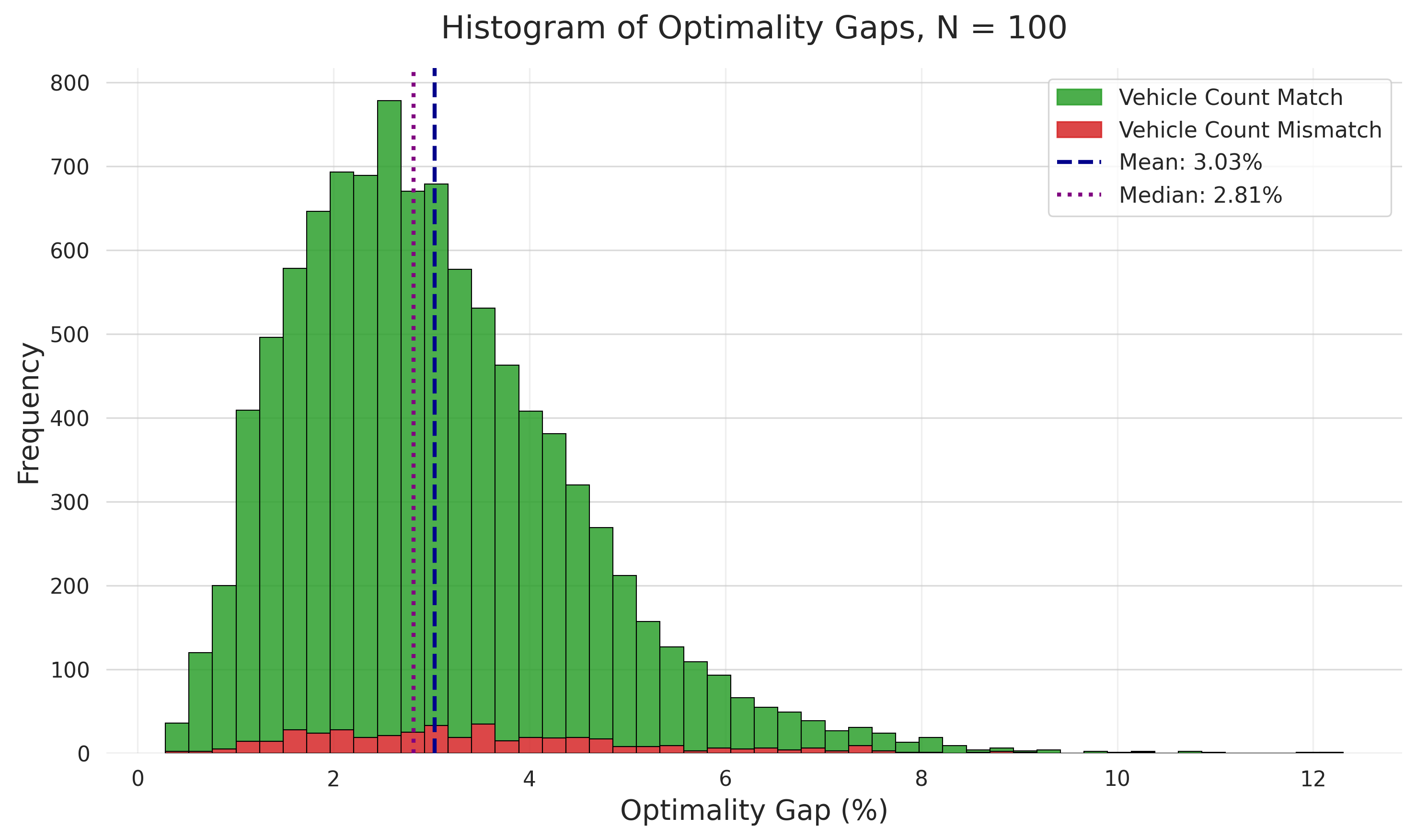}
    \hfill 
    \includegraphics[width=0.48\textwidth]{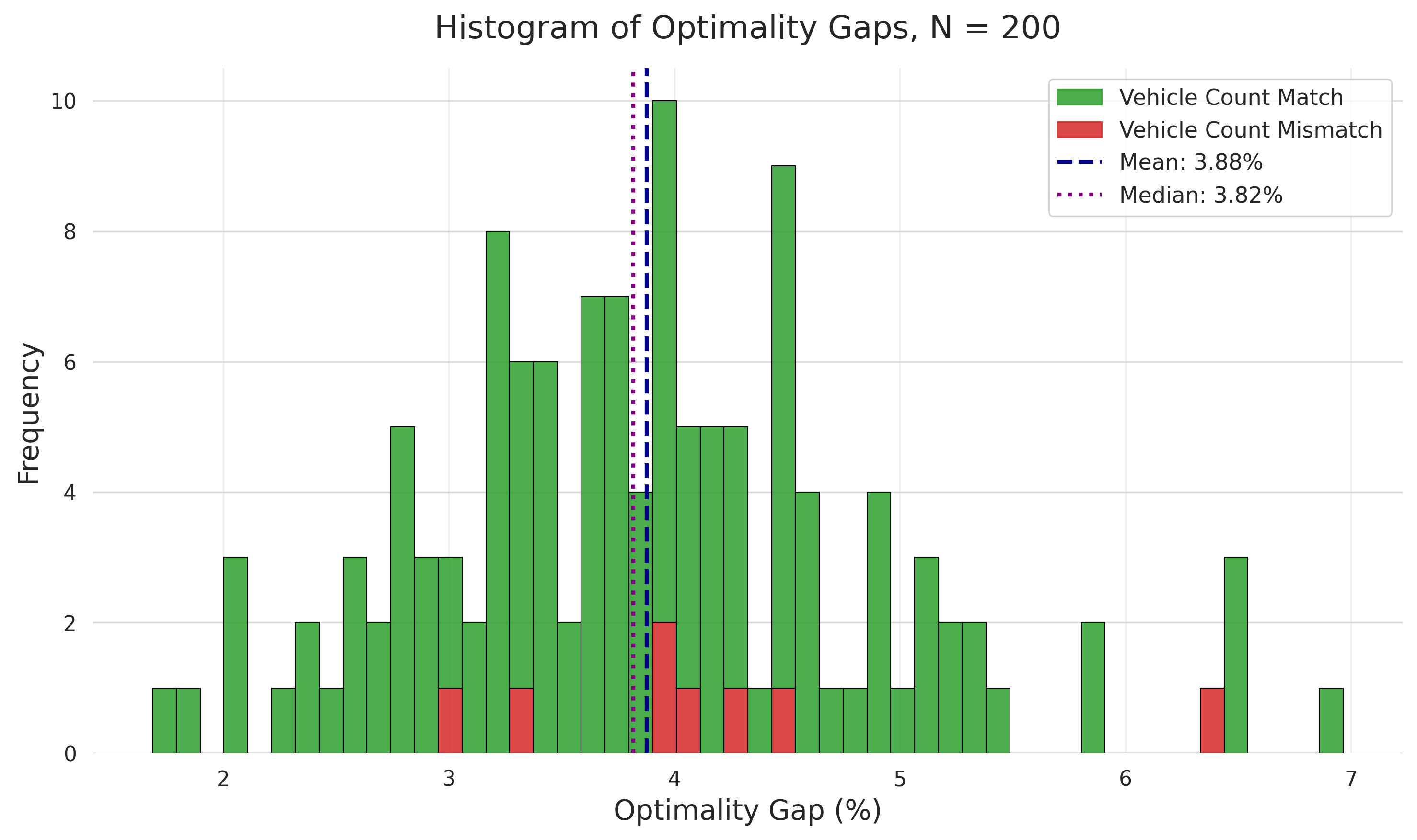}
    \caption{Distribution of optimality gaps across different problem sizes. Observe how $N=100$ exhibits a right skew, whereas $N=200$ resembles a normal distribution. A potential reason for the difference might be due to the sample size. Nonetheless, looking into the instances with the largest optimality gaps can inform future research.}
    \label{fig:optimality_gaps_histogram}
\end{figure}

\begin{figure}[ht]
    \centering
    \includegraphics[width=\textwidth]{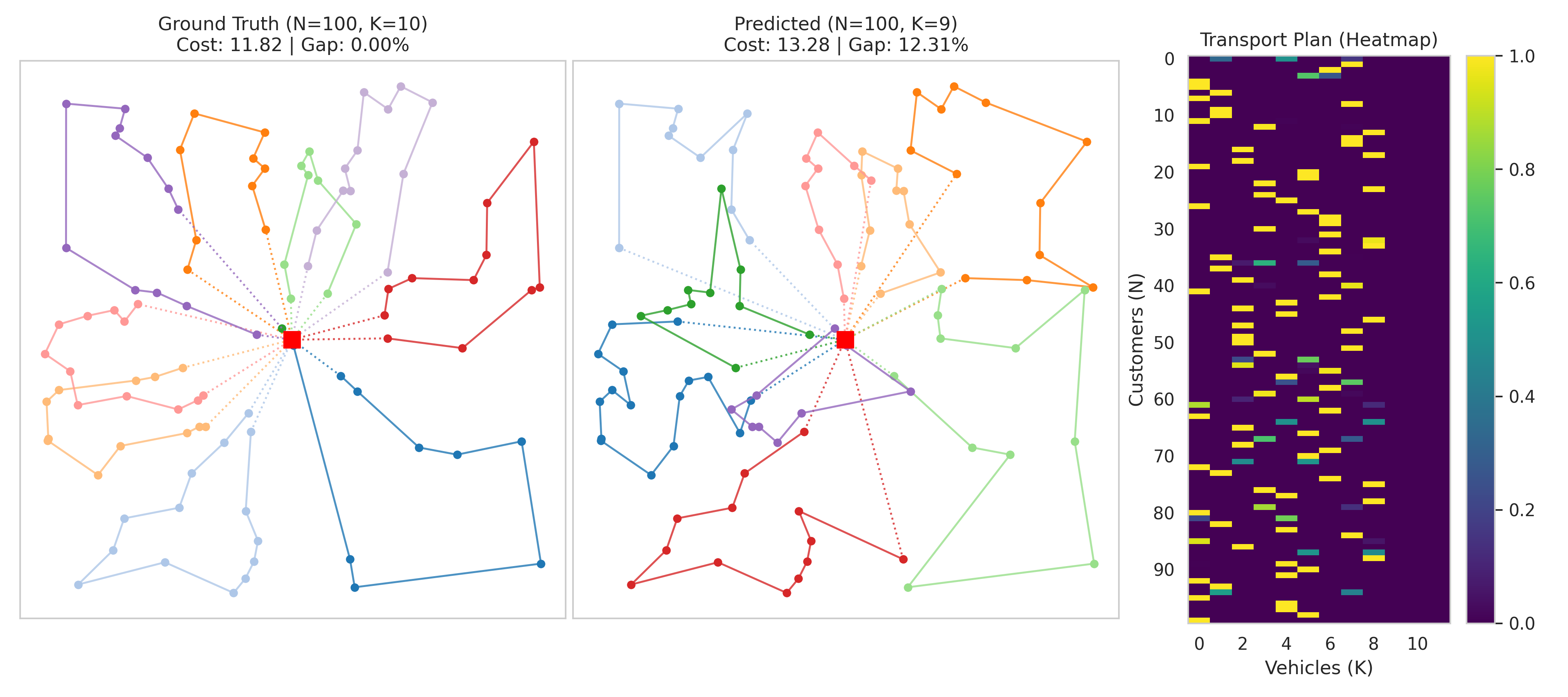}
    \includegraphics[width=\textwidth]{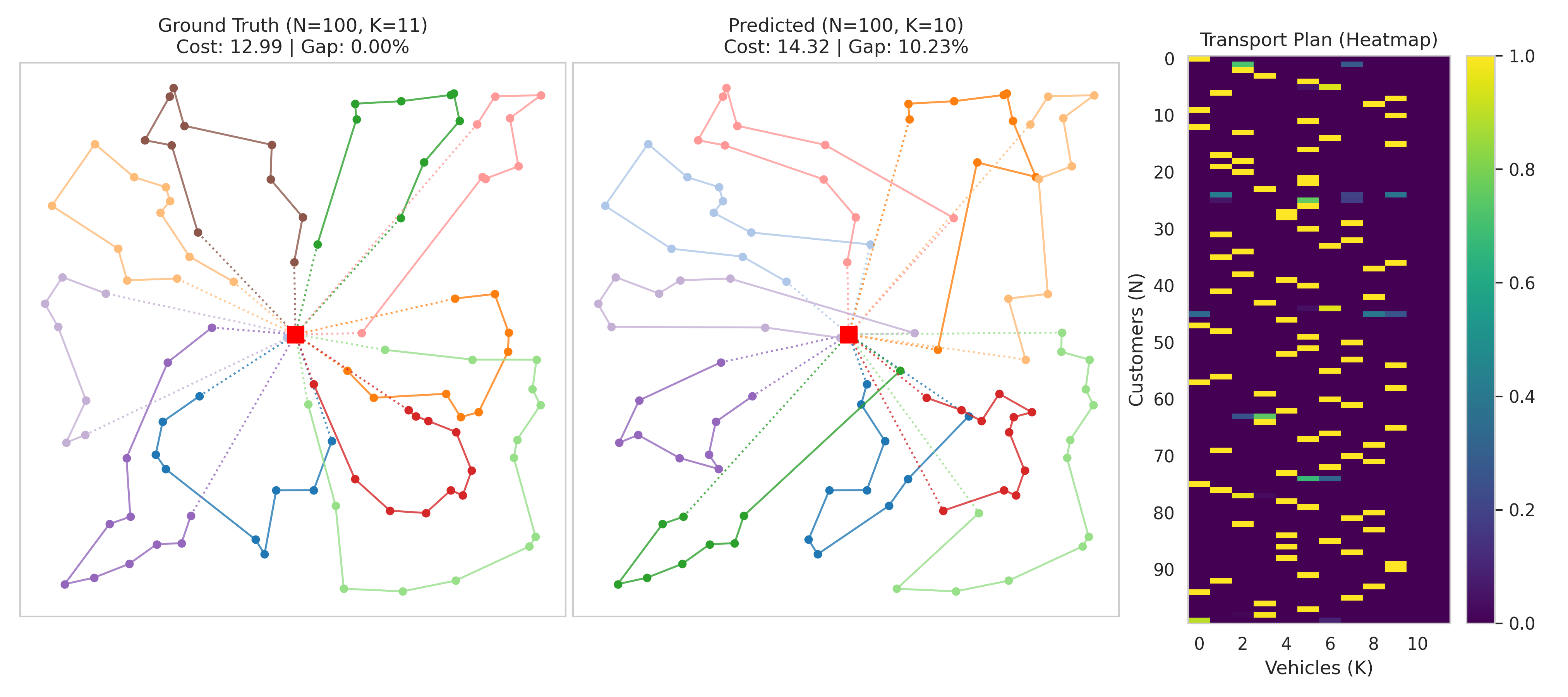}
    \caption{A selection from the ten solutions exhibiting the largest optimality gaps for $N=100$. In both examples, the \ac{HGS} solution returned a $K$ value of $K_{\min} + 1$. Despite the significant optimality gaps, the associated transport plans indicate high confidence across many of the assignments.}
    \label{fig:n100_worst_different_k}
\end{figure}

\begin{figure}[ht]
    \centering
    \includegraphics[width=\textwidth]{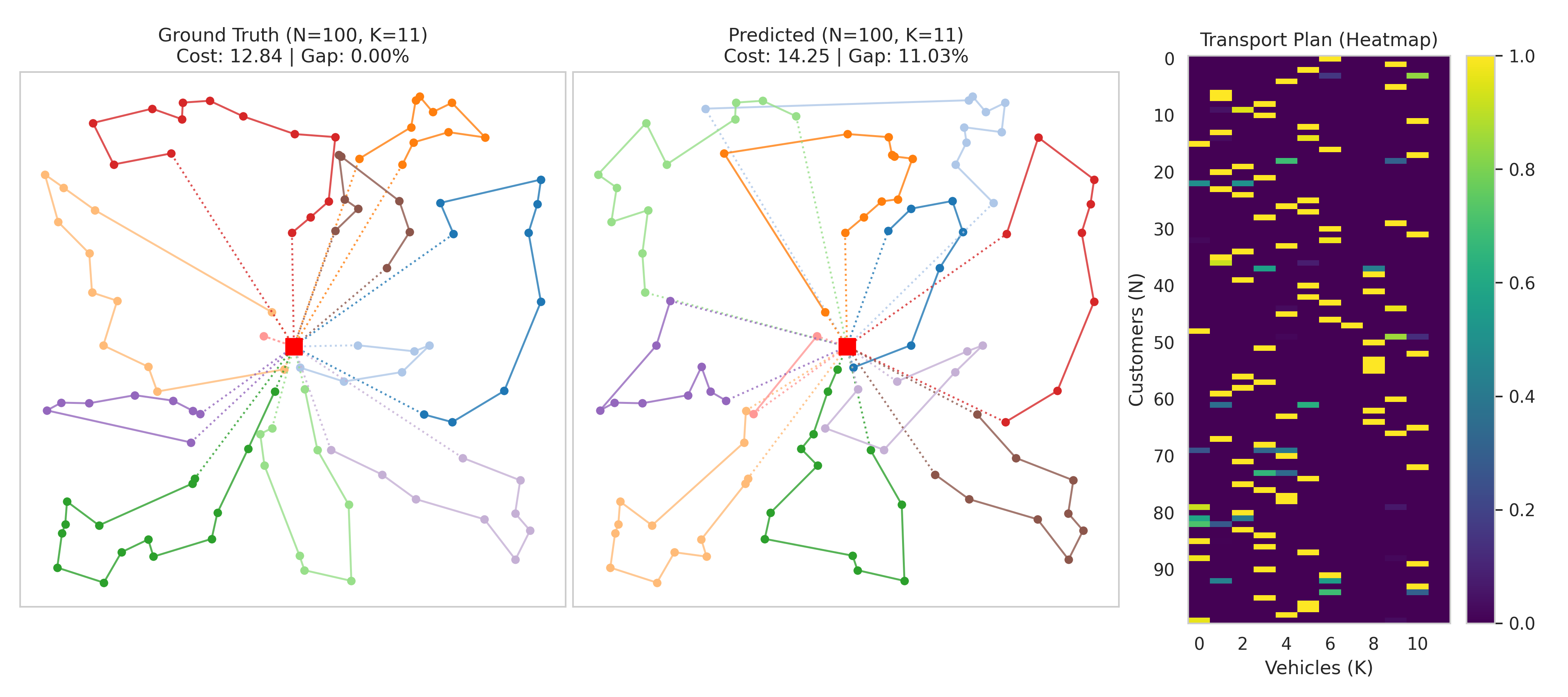}
    \includegraphics[width=\textwidth]{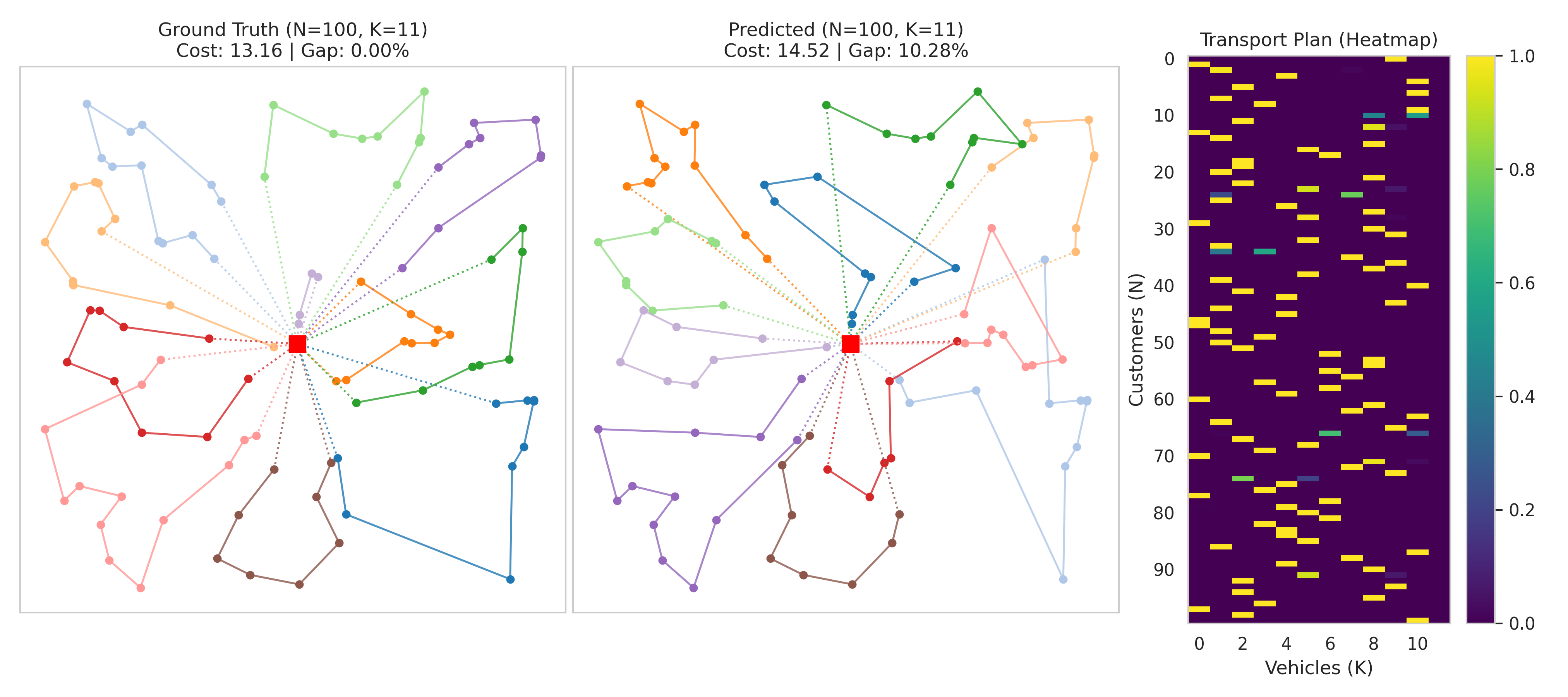}
    \caption{A selection from the ten solutions exhibiting the largest optimality gaps for $N=100$. Observe that both examples feature noticeably short routes with very few stops situated near the depot. Specifically, in the ground truth top image, the \textcolor{routepink}{pink route} consists of a single stop adjacent to the depot, while in the ground truth bottom image, the \textcolor{routelilac}{lilac route} contains only four stops close to the depot. Despite the significant optimality gaps, the associated transport plans indicate high confidence across many of the assignments.}
    \label{fig:n100_worst_short_routes}
\end{figure}

\begin{figure}[ht]
    \centering
    \includegraphics[width=\textwidth]{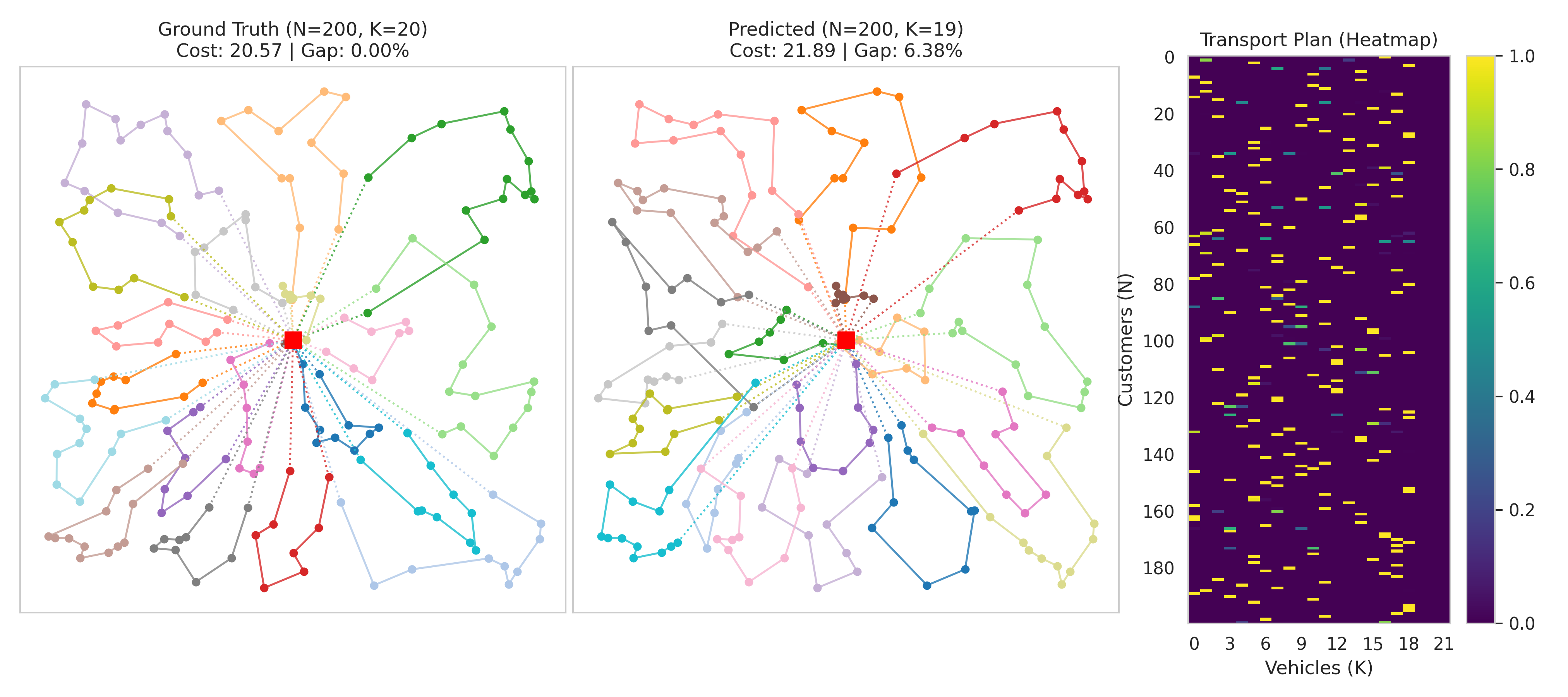}
    \caption{A selection from the ten solutions exhibiting the largest optimality gaps for $N=200$ problems. In general, as seen from the $N=200$ histogram in Figure \ref{fig:optimality_gaps_histogram}, low optimality gaps of under $7\%$ is generally achieved. In this example, the \ac{HGS} solution returned a $K$ value of $K_{\min} + 1$. In this case, the associated transport plan shows high confidence across many of the assignments and achieves an optimality gap of $6.38\%$.}
    \label{fig:n200_worst_different_k}
\end{figure}

\begin{table}[ht]
\centering
\caption{Performance comparison between strictly adhering to the theoretical minimum fleet size ($K_{\text{min}}$) versus taking the best solution found across $K_{\text{min}}$ and $K_{\text{min}} + 1$ vehicles. As discrete demand distributions can occasionally induce packing inefficiencies, expanding the solution space to include an additional vehicle captures these edge cases and yields strictly improved routing costs.}
\label{tab:k_comparison}
\resizebox{\linewidth}{!}{
\begin{tabular}{c c c l rr rr rr rr}
\toprule
\multirow{2}{*}{\textbf{Setup}} & \multirow{2}{*}{\textbf{Input}} & \multirow{2}{*}{$\mathcal{L}_{\text{BCE}}$} & \multirow{2}{*}{\textbf{Decoding}} & \multicolumn{2}{c}{CVRP100} & \multicolumn{2}{c}{CVRP200} & \multicolumn{2}{c}{CVRP500} & \multicolumn{2}{c}{CVRP1000} \\
\cmidrule(lr){5-6} \cmidrule(lr){7-8} \cmidrule(lr){9-10} \cmidrule(lr){11-12}
& & & & Gap & Time & Gap & Time & Gap & Time & Gap & Time \\
\midrule
\multirow{24}{*}{\rotatebox{90}{\textbf{Constant Capacity}}} & \multirow{12}{*}{$\Psi$} & \multirow{6}{*}{$\checkmark$} & Exact ($K_{\text{min}}$) & 3.030\% & 77.60 & 3.877\% & 5.41 & 4.540\% & 24.83 & 7.630\% & 359.49 \\
 &  &  & Exact (BB) & \textbf{3.007}\% & 150.07 & 3.786\% & 9.28 & 4.511\% & 41.79 & 4.311\% & 604.53 \\
\cmidrule{4-12}
 &  &  & Sparsified ($K_{\text{min}}$) & 3.239\% & 73.44 & 3.787\% & 5.12 & 4.458\% & 35.21 & 9.555\% & 410.24 \\
 &  &  & Sparsified (BB) & 3.098\% & 145.65 & 3.709\% & 9.10 & 4.429\% & 53.02 & 4.382\% & 722.47 \\
\cmidrule{4-12}
 &  &  & Hard Assign. ($K_{\text{min}}$) & 3.759\% & 75.30 & 4.866\% & 3.92 & 4.788\% & 49.27 & 3.768\% & 609.85 \\
 &  &  & Hard Assign. (BB) & 3.500\% & 142.57 & 4.570\% & 7.67 & 4.647\% & 70.81 & 3.699\% & 1136.42 \\
\cmidrule{3-12}
 &  & \multirow{6}{*}{$\times$} & Exact ($K_{\text{min}}$) & 3.161\% & 81.97 & 3.896\% & 5.65 & 4.382\% & 43.91 & 4.774\% & 172.39 \\
 &  &  & Exact (BB) & 3.132\% & 161.72 & 3.703\% & 12.24 & 4.330\% & 62.13 & 4.692\% & 284.20 \\
\cmidrule{4-12}
 &  &  & Sparsified ($K_{\text{min}}$) & 3.619\% & 73.04 & 4.964\% & 4.45 & 4.429\% & 37.73 & 13.082\% & 501.21 \\
 &  &  & Sparsified (BB) & 3.286\% & 143.92 & \textbf{3.675}\% & 8.94 & 4.389\% & 56.89 & 5.041\% & 926.52 \\
\cmidrule{4-12}
 &  &  & Hard Assign. ($K_{\text{min}}$) & 4.531\% & 68.87 & 4.953\% & 4.02 & 4.502\% & 50.05 & 4.054\% & 745.55 \\
 &  &  & Hard Assign. (BB) & 4.168\% & 135.99 & 4.744\% & 7.78 & 4.413\% & 72.51 & 3.914\% & 1208.47 \\
\cmidrule{3-12}
 & \multirow{12}{*}{$(x, y)$} & \multirow{6}{*}{$\checkmark$} & Exact ($K_{\text{min}}$) & 3.336\% & 80.77 & 4.209\% & 5.47 & 5.203\% & 19.62 & 8.369\% & 380.22 \\
 &  &  & Exact (BB) & 3.313\% & 159.20 & 4.085\% & 9.93 & 5.167\% & 34.41 & 4.907\% & 548.95 \\
\cmidrule{4-12}
 &  &  & Sparsified ($K_{\text{min}}$) & 3.424\% & 76.94 & 4.184\% & 5.04 & 5.217\% & 29.99 & 6.769\% & 397.26 \\
 &  &  & Sparsified (BB) & 3.326\% & 152.31 & 4.036\% & 9.38 & 5.176\% & 48.95 & 5.043\% & 689.43 \\
\cmidrule{4-12}
 &  &  & Hard Assign. ($K_{\text{min}}$) & 4.048\% & 74.37 & 4.740\% & 4.58 & 4.871\% & 107.01 & 4.279\% & 381.34 \\
 &  &  & Hard Assign. (BB) & 3.795\% & 145.77 & 4.567\% & 8.66 & 4.762\% & 120.77 & 4.149\% & 552.14 \\
\cmidrule{3-12}
 &  & \multirow{6}{*}{$\times$} & Exact ($K_{\text{min}}$) & 3.278\% & 102.03 & 3.857\% & 6.75 & 4.216\% & 41.38 & 4.073\% & 266.21 \\
 &  &  & Exact (BB) & 3.251\% & 198.98 & 3.715\% & 13.23 & 4.176\% & 69.27 & 3.957\% & 512.71 \\
\cmidrule{4-12}
 &  &  & Sparsified ($K_{\text{min}}$) & 3.512\% & 96.45 & 3.860\% & 7.05 & 5.580\% & 28.95 & 10.501\% & 481.56 \\
 &  &  & Sparsified (BB) & 3.327\% & 190.03 & 3.712\% & 13.59 & 4.149\% & 54.71 & 7.231\% & 955.83 \\
\cmidrule{4-12}
 &  &  & Hard Assign. ($K_{\text{min}}$) & 4.640\% & 89.76 & 5.223\% & 5.68 & 4.270\% & 108.49 & 3.262\% & 1240.40 \\
 &  &  & Hard Assign. (BB) & 4.281\% & 178.09 & 4.864\% & 11.34 & \textbf{4.145}\% & 135.81 & \textbf{3.180}\% & 2074.16 \\
\midrule
\multirow{24}{*}{\rotatebox{90}{\textbf{Standard Capacity}}} & \multirow{12}{*}{$\Psi$} & \multirow{6}{*}{$\checkmark$} & Exact ($K_{\text{min}}$) & 3.034\% & 87.69 & 4.830\% & 4.41 & 9.741\% & 6.92 & 26.375\% & 41.47 \\
 &  &  & Exact (BB) & \textbf{3.011}\% & 171.66 & 4.830\% & 8.81 & 9.712\% & 13.85 & 26.375\% & 82.94 \\
\cmidrule{4-12}
 &  &  & Sparsified ($K_{\text{min}}$) & 3.299\% & 76.35 & 4.749\% & 3.73 & 9.743\% & 9.63 & 28.242\% & 40.00 \\
 &  &  & Sparsified (BB) & 3.057\% & 150.24 & 4.749\% & 7.68 & 9.730\% & 18.85 & 28.242\% & 79.99 \\
\cmidrule{4-12}
 &  &  & Hard Assign. ($K_{\text{min}}$) & 3.758\% & 74.22 & 5.062\% & 3.58 & 8.268\% & 16.75 & 22.304\% & 37.90 \\
 &  &  & Hard Assign. (BB) & 3.504\% & 143.80 & 5.044\% & 7.11 & 8.247\% & 26.16 & 22.304\% & 74.75 \\
\cmidrule{3-12}
 &  & \multirow{6}{*}{$\times$} & Exact ($K_{\text{min}}$) & 3.174\% & 83.48 & 4.796\% & 4.11 & 9.046\% & 7.50 & 27.148\% & 20.81 \\
 &  &  & Exact (BB) & 3.144\% & 161.72 & 4.796\% & 8.09 & 9.042\% & 14.84 & 27.148\% & 41.65 \\
\cmidrule{4-12}
 &  &  & Sparsified ($K_{\text{min}}$) & 3.710\% & 76.43 & 4.767\% & 3.95 & 9.022\% & 7.22 & 29.759\% & 21.13 \\
 &  &  & Sparsified (BB) & 3.336\% & 152.07 & 4.766\% & 7.99 & 9.017\% & 14.84 & 29.759\% & 42.10 \\
\cmidrule{4-12}
 &  &  & Hard Assign. ($K_{\text{min}}$) & 4.555\% & 72.95 & 5.256\% & 3.77 & 8.013\% & 18.96 & 21.427\% & 19.78 \\
 &  &  & Hard Assign. (BB) & 4.191\% & 143.51 & 5.240\% & 7.64 & 8.013\% & 38.30 & 21.427\% & 39.14 \\
 \cmidrule{3-12}
 & \multirow{12}{*}{$(x, y)$} & \multirow{6}{*}{$\checkmark$} & Exact ($K_{\text{min}}$) & 3.316\% & 76.27 & 5.142\% & 3.58 & 10.487\% & 6.03 & 25.018\% & 18.71 \\
 &  &  & Exact (BB) & 3.296\% & 148.82 & 5.142\% & 7.25 & 10.487\% & 12.43 & 25.018\% & 37.68 \\
\cmidrule{4-12}
 &  &  & Sparsified ($K_{\text{min}}$) & 3.408\% & 74.26 & 5.092\% & 3.68 & 10.673\% & 6.58 & 26.305\% & 20.21 \\
 &  &  & Sparsified (BB) & 3.319\% & 144.79 & 5.092\% & 7.34 & 10.668\% & 12.95 & 26.305\% & 40.41 \\
\cmidrule{4-12}
 &  &  & Hard Assign. ($K_{\text{min}}$) & 4.054\% & 70.81 & 5.130\% & 3.34 & 9.024\% & 6.49 & 21.309\% & 17.79 \\
 &  &  & Hard Assign. (BB) & 3.806\% & 138.46 & 5.122\% & 6.63 & 9.024\% & 12.59 & 21.309\% & 35.58 \\
\cmidrule{3-12}
 &  & \multirow{6}{*}{$\times$} & Exact ($K_{\text{min}}$) & 3.283\% & 81.93 & 4.745\% & 3.97 & 8.438\% & 7.06 & 20.904\% & 20.28 \\
 &  &  & Exact (BB) & 3.252\% & 159.29 & 4.729\% & 7.83 & 8.432\% & 13.48 & 20.904\% & 40.69 \\
\cmidrule{4-12}
 &  &  & Sparsified ($K_{\text{min}}$) & 3.691\% & 77.36 & 4.694\% & 3.89 & 8.407\% & 7.85 & 23.744\% & 22.65 \\
 &  &  & Sparsified (BB) & 3.401\% & 151.54 & \textbf{4.669}\% & 7.79 & 8.401\% & 16.11 & 23.744\% & 44.41 \\
\cmidrule{4-12}
 &  &  & Hard Assign. ($K_{\text{min}}$) & 4.633\% & 71.96 & 4.877\% & 3.59 & 7.465\% & 8.56 & \textbf{18.952}\% & 19.31 \\
 &  &  & Hard Assign. (BB) & 4.291\% & 141.91 & 4.808\% & 7.24 & \textbf{7.446}\% & 16.56 & \textbf{18.952}\% & 39.19 \\
\bottomrule
\multicolumn{11}{l}{BB = Best of $K_{\text{min}}$ and $K_{\text{min}} + 1$.}
\end{tabular}
}
\end{table}

%% file: appendices/app_sparsification_efficiency.tex
\section{Sparsification Efficiency and CAP Optimality.}
\label{app:sparsification-efficiency}
As exact \ac{CAP} solutions become computationally intractable as routing instances scale, we leverage the continuous transport plan $\hat{Y}$ to sparsify the exact solver's search space. 
As illustrated in Figure \ref{fig:mip_precision_recall}, thresholding $\hat{Y}$ is highly calibrated: across $N \in \{100, 200\}$, a permissive threshold ($\tau=0.01$) retains $>95\%$ of optimal edges with $>80\%$ precision, while a strict threshold ($\tau=0.99$) correctly hard-assigns $>70\%$ of nodes with a minimal $\approx3\%$ error rate. 
This Sinkhorn-guided sparsification drastically restricts the exact solver to evaluating only the model's uncertain boundary states. 
Crucially, the resulting \ac{CAP} optimality gap exhibits a highly non-linear impact on the final \ac{CVRP} routing cost. 
As shown in Figure~\ref{fig:optimality-gap}, early termination of the exact search yields stagnant routing improvements, but breaching a 1\% CAP gap triggers a sharp phase transition manifesting as a steep plunge in routing optimality. 
Notably, on large $N=1000$ problems, models trained with $\mathcal{L}_{BCE}$ face initial computational resistance—requiring roughly 45 seconds to reach a 16\% CAP gap—but subsequently enable explosive convergence, needing only 10 additional seconds to locate highly optimal assignments. 
This delayed convergence suggests the learned manifold initially stalls but ultimately accelerates the integer programming search.

\begin{figure}[ht]
    \centering
    \includegraphics[width=0.9\textwidth]{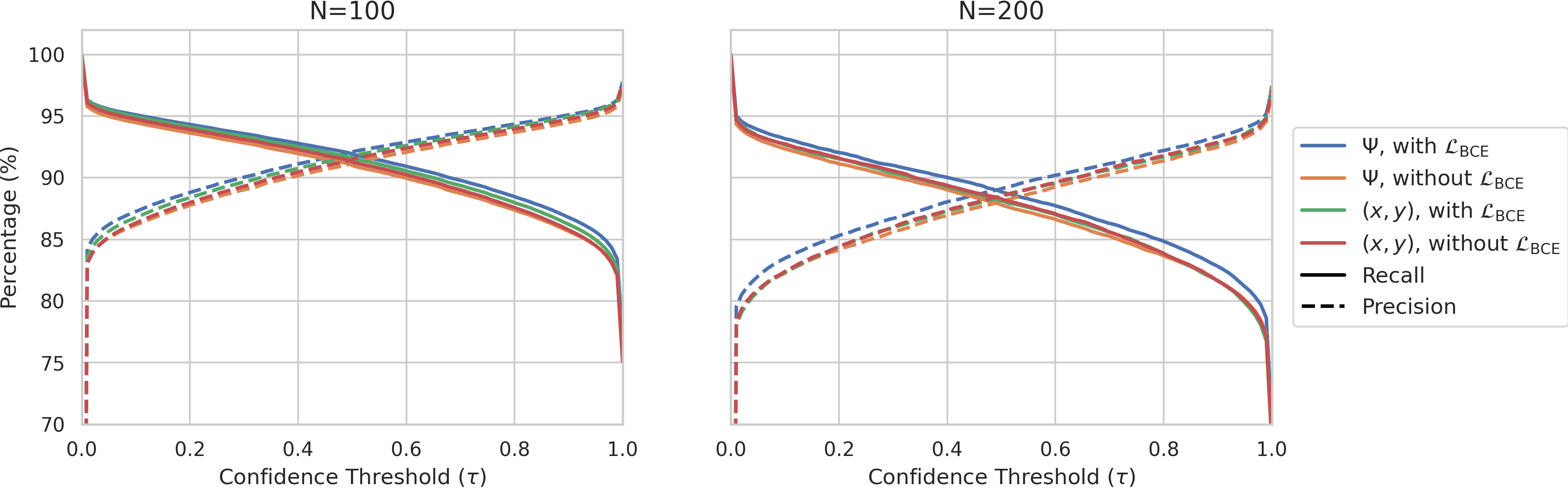}
    \caption{$\hat{Y}$ precision and recall when compared against $\bar{Y}_{\text{exact}}$. At $\tau=0^+$, recall and precision is above 95\% and 80\% respectively, indicating that $\approx5\%$ of optimal edges are left out. At $\tau=1^-$, recall and precision is above 70\% and 97\% respectively, indicating that $\approx3\%$ of nodes are assigned to suboptimal vehicles.}
    \label{fig:mip_precision_recall}
\end{figure}

\begin{figure}[htbp]
    \centering
    
    \includegraphics[width=0.9\textwidth]{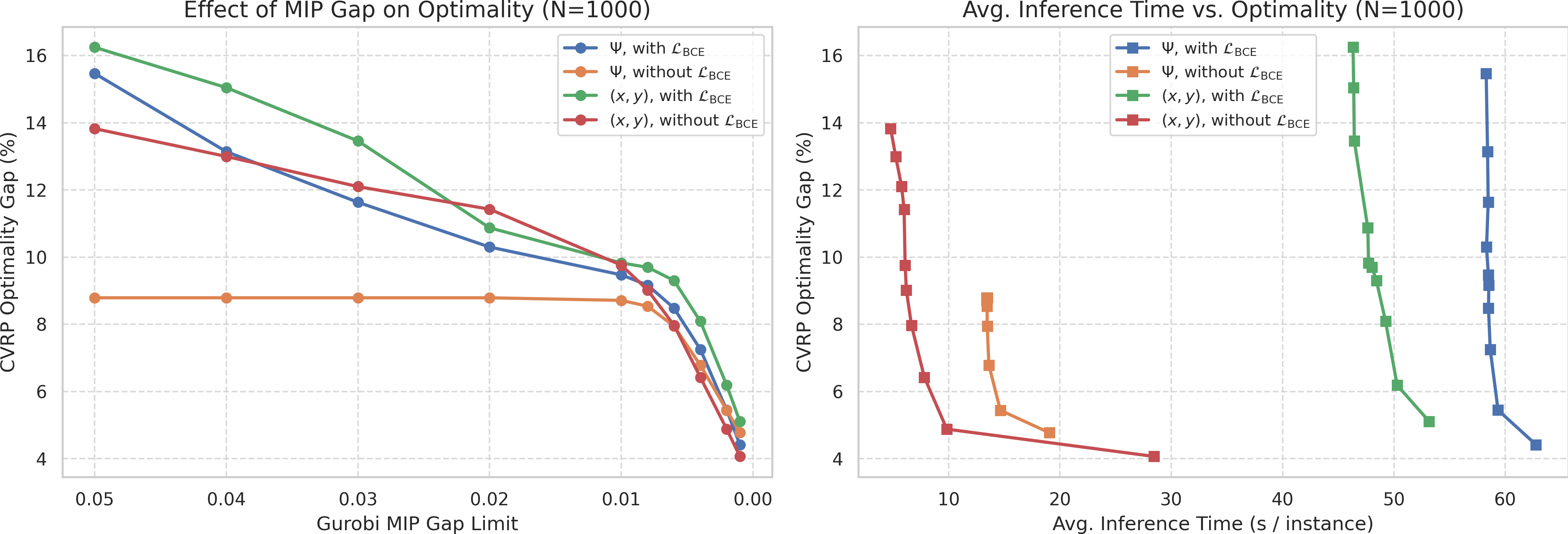}
    \caption{Pareto frontier of the solver's computational efficiency, mapping the MIP gap limit against both the final \ac{CVRP} optimality gap and the average total inference time (seconds/instance). The analysis evaluates out-of-distribution scaling ($N=1000$), comparing the spatial support embeddings ($\Psi$) against baseline raw coordinates ($(x,y)$), with and without $\mathcal{L}_{\text{BCE}}$.}
    \label{fig:optimality-gap}
\end{figure}